\documentclass{article}
\usepackage{amsmath, amsthm,amssymb}
\usepackage{microtype}
\usepackage{graphicx}
\usepackage{subfigure}
\usepackage{booktabs} 

\DeclareMathOperator*{\argmax}{arg\,max}

\usepackage{xcolor} 
\definecolor{blued}{RGB}{70,197,221}
\definecolor{pearOne}{HTML}{2C3E50}
\definecolor{pearTwo}{HTML}{A9CF54}
\definecolor{pearTwoT}{HTML}{C2895B}
\definecolor{pearThree}{HTML}{E74C3C}
\colorlet{titleTh}{pearOne}
\colorlet{bull}{pearTwo}
\definecolor{pearcomp}{HTML}{B97E29}
\definecolor{pearFour}{HTML}{588F27}
\definecolor{pearFith}{HTML}{ECF0F1}
\definecolor{pearDark}{HTML}{2980B9}
\definecolor{pearDarker}{HTML}{1D2DEC}
\usepackage{nameref}
\usepackage{hyperref}   
\hypersetup{
	colorlinks,
	citecolor=pearDark,
	linkcolor=pearThree ,
	urlcolor=pearDarker}
\usepackage{algorithm}
\usepackage{algorithmic}

\definecolor{graphicbackground}{rgb}{0.96,0.96,0.8}
\definecolor{rouge1}{RGB}{226,0,38}  
\definecolor{orange1}{RGB}{243,154,38}  
\definecolor{jaune}{RGB}{254,205,27}  
\definecolor{blanc}{RGB}{255,255,255} 
\definecolor{rouge2}{RGB}{230,68,57}  
\definecolor{orange2}{RGB}{236,117,40}  
\definecolor{taupe}{RGB}{134,113,127} 
\definecolor{gris}{RGB}{91,94,111} 
\definecolor{bleu1}{RGB}{38,109,131} 
\definecolor{bleu2}{RGB}{28,50,114} 
\definecolor{vert1}{RGB}{133,146,66} 
\definecolor{vert3}{RGB}{20,200,66} 
\definecolor{vert2}{RGB}{157,193,7} 
\definecolor{darkyellow}{RGB}{233,165,0}  
\definecolor{lightgray}{rgb}{0.9,0.9,0.9}
\definecolor{darkgray}{rgb}{0.6,0.6,0.6}
\definecolor{babyblue}{rgb}{0.54, 0.81, 0.94}
\definecolor{citrine}{rgb}{0.89, 0.82, 0.04}
\definecolor{misogreen}{rgb}{0.25,0.6,0.0}

\makeatletter
\newcommand{\numrel}[2]{
  \refstepcounter{equation}
  \ltx@label{#2}
  \overset{(\theequation)}{#1}
}
\newcommand{\numterm}[1]{\refstepcounter{equation} \ltx@label{#1} (\theequation)}
\makeatother

\let\originalleft\left
\let\originalright\right
\renewcommand{\left}{\mathopen{}\mathclose\bgroup\originalleft}
\renewcommand{\right}{\aftergroup\egroup\originalright}

\newcommand{\sset}[1]{\left\{#1\right\}}
\newcommand{\ceil}[1]{\left\lceil#1\right\rceil}
\newcommand{\floor}[1]{\left\lfloor#1\right\rfloor}

\newcommand{\II}[1]{\mathbb{I}{\left\{#1\right\}}}

\newcommand{\Bernoulli}{\mathrm{Bernoulli}}

 \newtheorem{theorem}{Theorem}

\newtheorem{remark}{Remark}
 \newtheorem{proposition}{Proposition}
\newtheorem{fact}{Fact}

\newcommand{\R}{\mathbb{R}}

\newcommand{\NN}{{\mathbb N}}

\newcommand{\EE}[1]{\mathbb{E}\left[#1\right]}

\newcommand{\EEcc}[2]{\mathbb{E}\left[\left.#1\right|#2\right]}

\newcommand{\PP}[1]{\mathbb{P}\left[#1\right]}
\newcommand{\Prb}{\mathbb{P}}

\newcommand{\pa}[1]{\left(#1\right)}

\newcommand{\norm}[1]{\left\|#1\right\|}

\newcommand{\abs}[1]{\left|#1\right|}
\newcommand{\imp}{\Rightarrow}

\newcommand{\CommaBin}{\mathbin{\raisebox{0.5ex}{,}}}

\newcommand{\transpose}{^\mathsf{\scriptscriptstyle T}}

\newcommand{\cF}{\mathcal{F}}

\newcommand{\cH}{\mathcal{H}}

\newcommand{\cO}{\mathcal{O}}

\newcommand{\cS}{\mathcal{S}}

\newcommand{\ba}{{\bf a}}

\newcommand{\bc}{{\bf c}}
\newcommand{\bC}{{\bf C}}

\newcommand{\be}{{\bf e}}

\newcommand{\bw}{{\bf w}}
\newcommand{\bW}{{\bf W}}

\renewcommand{\epsilon}{\varepsilon}

\renewcommand{\bar}{\overline}

\newcommand{\nothere}[1]{}

\usepackage{xspace}

\newcommand{\reach}[1]{\overset{#1}{\rightsquigarrow}}
\newcommand{\fixed}{c^\star_{0}}
\newcommand{\Fixed}[1]{C_{0,#1}}
\newcommand{\ret}[2]{\pa{#1 \vert #2}}
\newcommand{\vmeanw}[1]{\bar{\bw}_{#1}}
\newcommand{\meanw}[1]{\bar{w}_{#1}}

\newcommand{\meanc}[1]{\bar{c}_{#1}}

\newcommand{\counterc}[2]{N_{\!\ominus~\!\! #1,#2}}
\newcommand{\counterw}[2]{N_{\!\oplus~\!\! #1,#2}}
\newcommand{\bonus}{\text{Bonus}}
\newcommand{\sizecorr}[1]{\makebox[0cm]{\phantom{$\displaystyle #1$}}}

\usepackage[accepted]{icml2020}

\icmltitlerunning{Budgeted Online Influence Maximization}

\begin{document}

\twocolumn[
\icmltitle{Budgeted Online Influence Maximization}




\begin{icmlauthorlist}
\icmlauthor{Pierre Perrault}{adb,ens,inr}
\icmlauthor{Jennifer Healey}{adb}
\icmlauthor{Zheng Wen}{dm}
\icmlauthor{Michal Valko}{dm,inr}
\end{icmlauthorlist}

\icmlaffiliation{adb}{Adobe Research, San Jose, CA}
\icmlaffiliation{ens}{ENS Paris-Saclay}
\icmlaffiliation{inr}{Inria Lille}
\icmlaffiliation{dm}{DeepMind}

\icmlcorrespondingauthor{Pierre Perrault}{pierre.perrault@outlook.com}

\icmlkeywords{Online influence maximization, bandits}

\vskip 0.3in
]



\printAffiliationsAndNotice{} 

\begin{abstract}
We introduce a new budgeted framework for online  influence  maximization,   considering the total cost of an advertising campaign instead  of  the  common  cardinality  constraint on a chosen influencer set. Our approach models better the  real-world  setting  where  the cost of influencers varies and advertizers want to find the best value for their overall social advertising budget. We propose an algorithm assuming  an  independent  cascade  diffusion model  and  edge-level  semi-bandit  feedback, and provide both theoretical and experimental results.  Our analysis is also valid for the cardinality-constraint  setting  and  improves the state of the art regret bound in this case.
\end{abstract}
\vspace{-.5cm}
\section{Introduction}
\label{sec:intro}
\vspace{-.1cm}
Viral marketing through online social networks now represents a significant part of many digital advertising budgets.  In this form of marketing, companies incentivize chosen influencers in social networks (e.g., Facebook, Twitter, YouTube) to feature a product in hopes that their followers will adopt the product and repost the recommendation to their own network of followers.  
The effectiveness of the chosen set of influencers can be measured by the expected number of users that adopt the product due to their initial recommendation, called the \emph{spread}. 
\emph{Influence maximization} (IM, \citealp{kempe2003maximizing}) is the problem of 
choosing the optimal set of influencers to maximize the spread under a cardinality  constraint on the chosen set.

In order to define the spread, we need to specify a diffusion process such as independent cascade (IC) or linear threshold (LT) \citep{kempe2003maximizing}.  The parameters of these models are usually  \emph{unknown}.  Different methods exist to estimate the
parameters of the diffusion model from historical data (see section~\ref{related}) however historical data is often difficult to obtain. Another possibility is to consider \emph{online influence maximization} (OIM) \citep{vaswani2015influence,wen2017online} where an agent actively learns about the network by interacting with it repeatedly, trying to find the best seed influencers.  The agent thus faces the dilemma of exploration versus exploitation, allowing us to see it as \emph{multi-armed bandits} problem \citep{auer2002finite}. More precisely, the agent faces IM over $T$ rounds. Each round, it selects $m$ seeds (based on \emph{feedback} from prior rounds) and diffusion occurs; then it gains a reward equal to the spread and receives some feedback on the diffusion. 

IM and OIM optimize with the constraint of a fixed number of seeds.  This reflects a fixed seed cost model, for example, where influencers are incentivized by being given an identical free product.  In reality, however, many influencers demand different levels of compensation.  Those with a high out-degree (e.g., number of followers) are usually more expensive.  Due to these cost variations, marketers usually wish to optimize their seed sets $S$ under a budget $c(S)\leq b$ rather than a cardinality constraint $\abs{S}\leq m$.  Optimizing a seed set under a budget has been studied in the \emph{offline} case by \citet{nguyen2013budgeted}.
In the \emph{online} case, \citet{wang2020fast} considered the relaxed constraint $\EE{c(S)}\leq b$, where the expectation is over the possible randomness of $S$.\footnote{This relaxation is to avoid a computationally costly \emph{partial enumeration} \citep{krause2005note,khuller1999budgeted}}
We believe however that the constraint of a fixed, equal budget $c(S)\leq b$ at each round does not sufficiently model the willingness to choose a cost-efficient seed set. Indeed, we see that the choice of $b$ is crucial: a $b$ too large translates into a waste of budget (some seeds that are too expensive will be chosen) and a $b$ too small translates into a waste of time (a whole round is used to influence only a few users).  
To circumvent this issue, instead of a budget per round, in our framework, we allow the agent to choose seed sets of any cost at each round, under an overall budget constraint (equal to $B=bT$ for instance). 
In summary, we incorporate the OIM framework into a \emph{budgeted bandit} setting. 
Our setting is more flexible for the agent, and better meets real-world needs.

\subsection{Related work on IM}\label{related}
IM can be formally defined as follows. A social network is modeled as a directed graph 
$G=(V,E)$, with nodes $V$ representing users and edges $E$ representing connections. An underlying diffusion model $D$ governs how information spreads in $G$. More precisely, $D$ is a probability distribution on subgraphs $G'$ of $G$, and given some seed set $S$, the spread $\sigma\pa{S}$ is defined as the expected number of $S$-reachable\footnote{nodes that are reachable from some node in $S$.} nodes in $G'\sim D$.  IM aims to find $S$ that is a solution to 
\begin{align}\max_{\abs{S}=m} \sigma\pa{S}.\label{IM}\end{align}
Although IM is NP-hard under standard diffusion models --- i.e., IC and LT --- $\sigma$ is a monotone \emph{submodular}\footnote{$f$ is submodular if $f(A\!\cup\!\sset{i})\!-\!f(A)$ is non-increaing with $A$.} function \citep{fujishige2005submodular}, and given a value oracle access to $\sigma$, the standard \textsc{greedy} algorithm solves \eqref{IM} within a $1-1/e$ approximation factor \citep{Nemhauser1978}. There have been multiple lines of work for IM, including the development of heuristics, approximation algorithms, as well as alternative diffusion models \citep{Leskovec2007,6137225,tang2014influence,Tang2015}.  Additionally, there are also results on learning $D$ from data in the case it is not known \citep{Saito2008,Goyal2010,Gomez-Rodriguez2012,Netrapalli2012}. 

\subsection{Related work on OIM}
Prior work in OIM has mainly considered either \emph{node level semi-bandit} feedback \citep{vaswani2015influence}, where the agent observes all the $S$-reachable nodes in $G'$, or \emph{edge level semi-bandit} feedback \citep{wen2017online}, where the agent observes the whole $S$-reachable subgraph (i.e., the subgraph of $G'$ induced by $S$-reachable nodes).  Other, weaker, feedback settings have also been studied including:
pairwise influence feedback, where all nodes that would be influenced by a seed set are observed but not the edges connecting them, i.e., $\pa{\sset{i}\text{-reachable nodes}}_{i\in S}$ is observed \citep{vaswani2017model}; 
local feedback, where the agent observes a set of out-neighbors of $S$ \citep{carpentier2016revealing} and immediate neighbor observation where the agent only observes the out-degree of $S$ \citep{lugosi19a}. 

\subsection{Our contributions}
In this paper, we define the budgeted OIM paradigm and propose a performance metric for an online policy on this problem using the notion of \emph{approximation regret} \citep{chen13a}. 
To the best of our knowledge, the both of contributions are new.
We then focus our study on the IC model with edge level semi-bandit feedback. We design a \textsc{cucb}-style algorithm 
and prove logarithmic regret bounds. 
We also propose some modifications of this algorithm with improving the regret rates. 
These gains apply to the non-budgeted setting, giving an improvement over the state-of-the-art analysis of the standard \textsc{cucb}-approach \citep{wang2017improving}. Our proof incorporates an approximation guarantee of \textsc{greedy} for ratio of submodular and modular functions, which may also be of independent interest.

\section{Problem definition} 
In this section, we formulate the problem of budgeted OIM and give a regret definition for evaluating policies in that setting. We also justify our choice for this notion of regret. We typeset vectors in bold and indicate components with indices, e.g., for some set $I$, $\ba=(a_i)_{i\in I} \in \R^I$ is a vector on $I$. Let $\be_i$ be the $i^{th}$ canonical unit
vector of $\R^I$. The incidence vector of any subset $A\subset I$ is \(\be_A\triangleq \sum_{i\in A}\be_i.\)



We consider a fixed directed network $G=(V,E)$, known to the agent, with $V\triangleq \sset{1,\dots,\abs{V}}$. We denote by $ij\in E$ the directed edge from node $i$ to $j$ in $G$. We assume that $G$ doesn't have self-loops, i.e., for all $ij\in E,~i\neq j$. For a node $i\in V$, a subset $S\subset V$, and a vector $\bw\in \sset{0,1}^E$, the predicate $S\reach{\bw} i$ holds if, in the graph defined by $G_\bw\triangleq\pa{V,\sset{ij\in E, w_{ij}=1}}$, there is a forward path
from a node in $S$ to the node~$i$. If it holds, we say that $i$ is influenced by $S$ under~$\bw$.  
We define $p_i\pa{S;\bw}\triangleq\II{S\reach{\bw} i}$ and the \emph{spread} as $\sigma\pa{S;\bw}\triangleq \abs{\sset{i\in V,~S\reach{\bw} i}}$. 
Our diffusion process is defined by the random vector $\bW\in \sset{0,1}^E$, and our cost is defined by the random\footnote{Although costs are usually deterministic, we assume randomness for more generality (influencer campaigns may have uncertain surcharges for example).} 
vector $\bC\in [0,1]^{V\cup\sset{0}}$ where the added component $C_{0}$ represents any fixed costs\footnote{We provide a toy example where $C_{0}$ models a concrete quantity: Assume you want to fill your restaurant. You may pay some seeds and ask them to advertise/influence people. $C_{0}$ represents the cost of the food, the staff, the rent, the taxes, ...     }. Notice, random costs are neither assumed to be mutually independent nor independent from $\bW$. We will see that components of $\bW$ might however be mutually independent (e.g., for the IC model).
\vspace{-.1cm}
\subsection{Budgeted online influence maximization}
The agent interacts with the diffusion process across several rounds, using a learning policy. 
At each round $t\geq 1$, the agent first selects a seed set $S_t\subset V$, based on its past observations.  Then, the random vectors for both the diffusion process $\bW_t\sim\Prb_{\bW}$ and the costs  $\bC_{t}\sim\Prb_{\bC}$ are sampled independently from previous rounds. Then, the agent
observes some feedback from both the diffusion process and the costs. 

We provide in \eqref{cumreward} the expected cumulative rewards $F_{B}$ defined for some total budget $B>0$. The goal for the agent is to follow a learning policy $\pi$ maximizing $F_{B}$. In \eqref{cumreward}, recall that $S_t$ is the seed set selected by $\pi$ at round~$t$.

\vspace{-.6cm}
\begin{align} F_{B}\pa{\pi}\triangleq \EE{\sum_{t=1}^{\tau_B-1} { \sigma\pa{S_t;\bW_t}}}.\label{cumreward}
\end{align}\vspace{-.5cm}

$\tau_B$ is the random round at which the remaining
budget becomes negative: if $B_t \triangleq B -
\sum_{t'\leq t } \pa{\be_{S_{t'}}\transpose \bC_{t'}+\Fixed{t'}}
$, then $B_{\tau_B-1}\geq 0$ and 
$B_{\tau_B} < 0$. Notice, quantities
$B_t$ and $\tau_B$ are usual in budgeted multi-armed bandits 
\citep{xia2016budgeted,Ding2013}.

\vspace{-.1cm}
\subsection{Performance metric}
We restrict ourselves to \emph{efficient} policies, i.e., we consider a complexity constraint on the policy the agent can follow: For a round $t$, the space and time complexity for computing $S_t$ has to be polynomial in $\abs{V}$, and polylogarithmic in $t$.
To evaluate the performance of a learning policy $\pi$, we use the notion of approximation regret \citep{kakade2009playing,streeter2009online,Chen2015combinatorial}. The agent wants to follow a learning
policy $\pi$ which minimizes
\begin{align*}
R_{B,\varepsilon}\pa{\pi}\triangleq \pa{1-{1}/{e}-\varepsilon}F^\star_{B} - F_{B}\pa{\pi},
\end{align*}
where $F^\star_{B}$ is the best possible value of $F_{B}$ over all policies (thus leveraging on the knowledge of $\Prb_{\bW}$ and $\Prb_{\bC}$), and where $\varepsilon>0$ is some parameter the agent can control to determine the tradeoff between computation
and accuracy.

\begin{remark}
This OIM with a total budget $B$ is different from OIM in previous work, such as \citet{wang2017improving}, even when we set all costs to be equal. In our setting, there is only one total budget for all rounds, and the policy is free to choose seed sets of different cost in each round, whereas in the previous work, each round had a fixed budget for the number/cost of seeds selected. Our setting thus avoid the use of a budget per round, which is in practice more difficult to establish than a global budget $B$. Nevertheless, as we will see in section~\ref{sec:budget}, both types of constraints (global and per round) can be considered simultaneously when the true costs are known. 
\end{remark}



\subsection{Justification for the approximation regret}
In the non-budgeted OIM problem with a cardinality constraint given by $m\in [\abs{V}]$, let us recall that the approximation regret 
\vspace{-.1cm}\begin{align*}R_{T,\varepsilon}\pa{\pi}\!\triangleq\!\sum_{t\leq T}\max_{\overset{S\subset V,}{\abs{S}= m}}\EE{\pa{1\!-\!{1}/{e}\!-\!\varepsilon}\sigma\pa{S;\bW}\!-\!\sigma\pa{S_t;\bW}}\end{align*}\vspace{-.5cm}

is standard  \citep{wen2017online,wang2017improving}. In this notion of regret, the factor $\pa{1-{1}/{e}-\varepsilon}$ \citep{feige1998threshold,chen2010scalable} reflects the difficulty of approximating the following NP-Hard \citep{kempe2003maximizing} problem
in the case the distribution $\Prb_{\bW}$ is described by IC or LT, and is known to the agent: \begin{align}\max_{S\subset V,~\abs{S}= m}\EE{\sigma\pa{S;\bW}}.\label{np-hard}\end{align}
For our budgeted setting, at first sight, it is not straightforward to know which approximation factor to choose. Indeed, since the random horizon may be different in $F_B\pa{\pi}$ and in $F_B^\star$, the expected regret $R_{B,\varepsilon}\pa{\pi}$ is not expressed as the expectation of a sum of approximation gaps, so we can't directly reduce the regret level approximability to the gap level approximability.
We thus consider a quantity provably close to $R_{B,\varepsilon}\pa{\pi}$ and easier to handle.
\begin{proposition}
Define \[\lambda^\star\triangleq \max_{S\subset V }\frac{\EE{\sigma\pa{S;\bW}}}{\EE{\be_{S\cup\sset{0}}\transpose \bC}}\cdot\] 
\vspace{-.55cm}

For all  $S\subset V,$ define the \emph{gap} corresponding to $S$ as 
\vspace{-.1cm}
\[ \Delta\pa{S}\triangleq  \pa{1-{1}/{e}-\varepsilon}\lambda^\star\EE{\be_{S\cup\sset{0}}\transpose \bC} - \EE{\sigma\pa{S;\bW}}. \]
Then, for any policy $\pi$ selecting $S_t$ at round $t$,\vspace{-.6cm}

\[\abs{R_{B,\varepsilon}\pa{\pi} - \EE{\sum_{t=1}^{\tau_B-1} \Delta\pa{S_t}{}}}\leq 2\abs{V}+2\lambda^\star\pa{1 +\abs{V}}.\] 
\label{prop:budgeted_regret}
\end{proposition}\vspace{-.5cm}
From Proposition~\ref{prop:budgeted_regret}, whose proof can be found in Appendix~\ref{app:budgeted_regret}, $R_{B,\varepsilon}\pa{\pi}$ and $\EE{\sum_{t=1}^{\tau_B-1} \Delta\pa{S_t}{}}$ are equivalent in term of regret upper bound rate. Therefore, 
the factor $\pa{1-{1}/{e}-\varepsilon}$ 
should reflect the approximability of
\begin{align}\max_{S\subset V}{\EE{\sigma\pa{S;\bW}}}/{\EE{\be_{S\cup\sset{0}}\transpose \bC}}=\max_{S\subset V} f(S)/c(S).\label{rel:approxrat}\end{align}
Considering the specific problem where the cost function is of the form $c(S) = c_1\abs{S}+c_0$, for some $\pa{c_0,c_1}\in [0,1]^2,$ we can reduce the approximability of \eqref{rel:approxrat} to the approximability of the following problem considered in \citet{wang2020fast}:
\begin{align}\max_{S\subset V}{\EE{f(S)}}~s.t.~{\EE{\abs{S}}\leq m},\label{rel:approxexpcard}\end{align}
for some given integer $m$, where the expectations are with respect to a randomization in the approximation algorithm. \citet{wang2020fast} proved that this problem is NP-hard by reducing to the \emph{set cover} problem. We show here that 
an approximation ratio $\alpha$ better that $1-1/e$ yields an approximation for set cover within $(1-\delta)\log(\abs{V})$, $\delta>0$, which is impossible unless $\mathrm{NP}\subset \mathrm{BPTIME}\pa{n^{\cO(\log \log(\abs{V}))}}$
\citep{feige1998threshold}. 
Consider the graph where the collection of out-neighborhoods is exactly the collection of sets in the set cover instance.
First, trying out all possible values of $m$, we concentrate on the case in which the optimal $m$ for set cover is tried out. As in \citet{feige1998threshold}, for $k\in \NN^*$, we repeatedly  apply the algorithm that $\alpha$-approximate \eqref{rel:approxexpcard}. It outputs a set $S_k$ (that can be associated with a set of neighborhoods) and after  each  application  the  nodes  already  covered  by  previous  applications  are removed from the graph, giving a sequence of objective functions $(f_k)$ with $f_1=f$. We thus obtain
\[\EEcc{f_k(S_k)}{S_1,\dots,S_{k-1}}\geq \alpha \pa{\abs{V}-\sum_{k'=1}^{k-1}f_{k'}(S_{k'})}.\]
Noticing that $\EE{f(S_1\cup\dots\cup S_k)}=\sum_{k'=1}^{k}\EE{f_{k'}(S_{k'})}$, we get
\[\EE{f(S_1\cup\dots\cup S_k)}\geq \pa{1-(1-\alpha)^k}\abs{V}.\]
After $\ell=\floor{\log(1/\abs{V})/\log(1-\alpha)}< (1-\delta)\log(\abs{V})$ iterations, we obtain that $S=S_1\cup\dots\cup S_{\ell}$ is a cover, i.e., $f(S)=\abs{V}$. The result follows noticing that 
 in expectation (and so with positive probability), we have $\abs{S}\leq \ell m$. 
\section{Algorithm for IC with edge level semi-bandit feedback}
\subsection{Setting}
For $\bw\in [0,1]^V$, we recall that we can define an IC model by taking 
\(\Prb_{\bW} = \otimes_{ij \in E}\Bernoulli\pa{w_{ij}}.\) 
We can extend the two previous functions $p_i$ and $\sigma$ to $\bw$ taking values in  $[0,1]^V$ as follows:  Let 
$\bW\sim\otimes_{ij \in E}\Bernoulli\pa{w_{ij}}$.
We define the probability that $i$ is influenced by $S$ under $\bW$ as $p_i\pa{S;\bw}\triangleq\PP{S\reach{\bW} i}$, and we let the spread be $\sigma\pa{S;\bw}\triangleq \EE{\abs{\sset{i\in V,~S\reach{\bW} i}}}$. 
 Another expression for the spread is $\sigma\pa{S;\bw}=\sum_{i\in V}p_i\pa{S;\bw}.$
We fix a weight vector on $E$, $\bw^\star \triangleq \pa{w^\star_{ij}}_{ij \in E}\in [0,1]^E$, a cost vector on $V\cup\sset{0}$, $\bc^\star\triangleq \pa{c^\star_{i}}_{i \in V\cup\sset{0}}\in [0,1]^{V\cup\sset{0}}$, with $\fixed>0$. These quantities are initially \emph{unknown} to the agent. We assume from now that 

\vspace{-.6cm}
\[\Prb_{\bW} \triangleq \otimes_{ij \in E}\Bernoulli\pa{w_{ij}^\star},\]
\vspace{-.9cm}

and that \[\EE{\bC}=\bc^\star.\]
\vspace{-.7cm}

We also define $S^\star\in\argmax_{S\subset V}{{\sigma\pa{S;\bw^\star}}}/{{\be_{S\cup\sset{0}}\transpose \bc^\star}}$.
We assume that the feedback received by the agent at round $t$ is $\sset{W_{ij,t},~ij\in E,~S_t\reach{\bW_t} i}$. The agent also receives semi-bandit feedback from the costs, i.e., $\sset{C_{i,t},~i\in V,~i\in S_t\cup\sset{0}}$ is observed.
\vspace{-.1cm}

\subsection{Algorithm design}\vspace{-.1cm}
In this subsection, we 
present \textsc{boim-cucb}, \textsc{cucb} for Budgeted OIM problem as Algorithm~\ref{algo:CUCB}. As we saw in Proposition~\ref{prop:budgeted_regret}, the policy that, at each round, $(1-1/e-\varepsilon)-$approximately maximize 

\vspace{-.7cm}
\begin{align}S\mapsto\frac{{\sigma\pa{S;\bw^\star}}}{{\be_{S}\transpose \bc^\star+\fixed}}\label{obj}\end{align}
\vspace{-.5cm}

has a bounded regret. Thus, \textsc{boim-cucb} 
shall be based on this objective.
Not only there are some estimation concerns due to the unknown parameters $\bw^\star,\bc^\star$, but in addition to that, we also need to evaluate/optimize our estimates of \eqref{obj}.

We begin by introducing some notations. We define the empirical means for $t\geq 1$ as: For all $i\in V\cup\sset{0}$,
\vspace{-.4cm}

\[\meanc{i,t-1}\triangleq\frac{\sum_{t'\in[t-1]}\II{i\in S_{t'}\cup\sset{0}}C_{i,t'}}{\counterc{i}{t-1}}\CommaBin\]
and for all $ij\in E$,
\vspace{-.6cm}

\[\meanw{ij,t-1}\triangleq\frac{\sum_{t'\in[t-1]}\II{S_{t'}\reach{\bW_{t'}}i}W_{ij,t'}}{\counterw{i}{t-1}}\CommaBin\]
where $\counterc{i}{t-1}\triangleq\sum_{t'=1}^{t-1}\II{i\in S_{t'}},~ \counterw{i}{t-1}\triangleq\sum_{t'=1}^{t-1}\II{S_{t'}\reach{\bW_{t'}}i}.$ 
Using concentration inequalities, we get confidence intervals for the above estimates. We  are then able to use an upper-confidence-bound (UCB) strategy \citep{auer2002finite}. More precisely, in the case costs are unknown, we first build the \emph{lower} confidence bound (LCB) on $c_i^\star$ as follows\vspace{-.3cm}
 \[c_{i,t}\triangleq 0\vee{ \pa{\meanc{i,t-1} - \sqrt{\frac{1.5\log\pa{t}}{\counterc{i}{t-1}}}}}. \]
 \vspace{-.4cm}
 
 We can also define UCBs for $w_{ij}^\star$:\vspace{-.2cm}
 \[w_{ij,t}\triangleq{1\wedge \pa{\meanw{ij,t-1} + \sqrt{\frac{1.5\log\pa{t}}{\counterw{i}{t-1}}}}}\cdot \]
 We use $w_{ij,t}=1$ (and $c_{i,t}=0$) when the corresponding counter is equal to $0$.
 Our \textsc{{boim-cucb}} approach chooses at each round $t$ the seed set $S_t$ given by Algorithm~\ref{algo:greedy}  which, as we shall see, approximately maximize $S\mapsto{{\sigma\pa{S;\bw
_t}}}/{\pa{\be_{S\cup\sset{0}}\transpose \bc_t}}$. 
Indeed, with high probability, this set function is an upper bound on the true ratio \eqref{obj} (using that $\sigma$ is non decreasing w.r.t. $\bw$).
Notice that this approach is followed by  \citet{wang2017improving} for the non budgeted setting, i.e., they choose $S_t,~\abs{S_t}\leq m$ that approximately maximize $S\mapsto{{\sigma\pa{S;\bw_t}}}$. To complete the description of our algorithm, we need to describe Algorithm~\ref{algo:greedy}. This is the purpose of the following.

 \begin{algorithm}[t]
\begin{algorithmic}
\STATE \textbf{Input}: $\varepsilon>0$, $B_0=B>0$. 
\FOR{each round $t\geq 1$}
\STATE If true costs are known, then $\bc_t \leftarrow \bc^\star$.
\STATE Compute $S_t$ given by Algorithm~\ref{algo:greedy} with input $S\mapsto\sigma\pa{S;\bw
_t},$ $\bc_t$. 
\STATE Select seed set $S_t$, and pay $\be_{S_t\cup\sset{0}}\transpose \bC_t$ (i.e., remove this cost from $B_{t-1}$ to get the new budget $B_t$).
\IF{$B_{t}\geq 0$,}
\STATE Get the reward $\sigma(S_t;\bW_t)$, get the feedback, and update corresponding quantities accordingly.
\ELSE \STATE The budget is exhausted: leave the for loop.
\ENDIF
\ENDFOR
\end{algorithmic}
\caption{\textsc{{boim-cucb}}\label{algo:CUCB}}
\end{algorithm}

\subsection{Greedy for ratio maximization}

In \textsc{boim-cucb}, one has to approximately maximize the ratio $S\mapsto {\sigma\pa{S;\bw
_t}}/{{\be_{S\cup \sset{0}}\transpose \bc_t}}$, that is a ratio of submodular over modular function. A \textsc{greedy} technique can be used (see Algorithm~\ref{algo:greedy}). Indeed, instead of maximizing the marginal contribution at each
time step, as the standard \textsc{greedy} algorithm do, the approach is to maximize the so called bang-per-buck, i.e., the
marginal contribution divided by the marginal cost. This builds a sequence of increasing subsets, and the final output is the one that maximizes the ratio.    
We prove in Appendix~\ref{app:greedy} the following Proposition~\ref{prop:greedy}, giving an approximation factor of $1-1/e$ for Algorithm~\ref{algo:greedy}. 

\begin{proposition}
\label{prop:greedy}
Algorithm~\ref{algo:greedy} with input $\sigma,\bc$ is guaranteed to obtain a solution $S$ such that:
\vspace{-.1cm}
\[\pa{1-e^{-1}}\frac{\sigma\pa{S^\star}}{\be_{S^\star\cup\sset{0}}\transpose \bc}\leq \frac{\sigma\pa{S}}{\be_{S\cup\sset{0}}\transpose \bc}\cdot\]
\end{proposition}
\vspace{-.3cm}
Notice, a similar result as Proposition~\ref{prop:greedy} is stated in Theorem~3.2 of~\citet{bai2016algorithms}. However, their proof doesn't hold in our case, since 
  their inequality~(16) would be true only for a \emph{normalized} cost (i.e. $c_0=0$). Actually, $c_0=0$ implies that $S^\star$ is a singleton, from subadditivity of~$\sigma$.

  \begin{algorithm}[t]
\begin{algorithmic}
\STATE \textbf{Input}: $\sigma$ that is an increasing submodular function,
\STATE ~~~~~~~~~~~ $\bc\in [0,1]^{V\cup\sset{0}}$.
\STATE $S_0\leftarrow\emptyset$.
\STATE $\rho \leftarrow [(\infty,i)]_{i\in V}$. 
\FOR{$k\in [|V|]$ 
}
\STATE $\texttt{checked}\leftarrow S_{k-1}$.
\STATE $(*)$ Remove the first element $\rho[0]=(\sim,i)$ from $\rho$.
\IF{$i\notin \texttt{checked}$}
\STATE Insert $(\pa{\sigma\pa{\sset{i}\cup S_{k-1}}-\sigma\pa{S_{k-1}}}/{c_{i}},i)$ in $\rho$, such that $\rho[:][0]$ is sorted in decreasing order.
\STATE Add $i$ to $\texttt{checked}$ and go back to $(*)$.
\ELSE
\STATE $S_k\leftarrow S_{k-1}\cup\sset{i}.$
\ENDIF
\ENDFOR
\STATE $k'\leftarrow \argmax_{k\in \sset{0,\dots,\abs{V}}} {\sigma\pa{S_k}}/{\be_{S_{k}\cup\sset{0}}\transpose \bc}.$
\STATE \textbf{Output}: $S_{k'}$.
\end{algorithmic}
\caption{\textsc{greedy} for ratio, Lazy implementation}\label{algo:greedy}
\end{algorithm}
  
For more efficiency, we use a greedy algorithm with lazy evaluations \citep{minoux1978accelerated,Leskovec2007}, leveraging on the \emph{submodularity} of $\sigma$. More precisely, in Algorithm~\ref{algo:greedy}, instead of taking the $\argmax$ in the step \[S_k\leftarrow S_{k-1}\cup\sset{\argmax_{i\in V\backslash S_{k-1}} \frac{\sigma\pa{\sset{i}\cup S_{k-1}}-\sigma\pa{S_{k-1}}}{c_{i}}},\] 
 we
maintain an upper bound $\rho$ (initially $\infty$)  on hhe marginal gain, sorted in decreasing order.  In each iteration $k$,  we  evaluates  the  element on  top  of  the  list,  say $i$,  and  updates  its  upper  bound with the marginal gain at $S_{k-1}$.   If after the update the upper bound is greater than the others, submodularity guarantees that $i$ is the element with the largest marginal gain. 

Algorithm~\ref{algo:greedy} (and the approximation factor) can't be used directly in the OIM context, since computing the exact spread $\sigma$ is \#P hard \citep{chen2010scalable}. However, with Monte Carlo (MC) simulations, it can efficiently reach an arbitrarily close ratio of $\alpha=1-1/e-\varepsilon$, with a high probability $1-{1}/\pa{t\log^2\pa{t}}$ \citep{kempe2003maximizing}. 

  \subsection{An alternative to Lazy Greedy: Ratio maximization from sketches}

In the previous approach, MC can still be computationally costly, since the marginal contribution have to be re-evaluated each time, using directed reachability computation in each MC instance.
There exist efficient alternatives to $\alpha-$approximately maximize the spread with cardinality constraint, such as \textsc{tim} from \citet{tang2014influence} and \textsc{skim} from \citet{cohen2014sketch}.
Adapting \textsc{tim} to our ratio maximization context is not straightforward, since it require to know the seed set size in advance, which is not the case in Algorithm~\ref{algo:greedy}. \textsc{skim} is more promising since it uses the standard \textsc{greedy} in a \emph{sketch space}.
We provide an adaptation of \textsc{skim} for approximately maximize the ratio (see Appendix~\ref{app:sketch}). It uses the 
bottom-$k$ min-hash sketches of \citet{cohen2014sketch}, with a threshold for the length of the sketches that depends on both  $k=\floor{\varepsilon^{-2}\log\pa{1/\delta}}$ and the cost, where $\delta$ is an upper bound on the probability that the relative error is larger than $\varepsilon$. Exactly as \citet{cohen2014sketch} proved the approximation ratio for \textsc{skim}, this approach reaches a factor
 of $1-1/e-\varepsilon$, with high probability.
 More precisely, at round $t\geq 2$, we can actually choose to have the approximation with $\delta = {1}/\pa{t\log^2\pa{t}},$
only adding a $O\pa{\log(t)}$ factor in the computational complexity of Algorithm~\ref{algo:greedy} \citep{cohen2014sketch}, thus remaining efficient.\footnote{\citet{wang2017improving} uses the notion of approximation regret with a certain fixed probability. Here, we rather fix this probability to $1$ and allow for a $O\pa{\log(t)}$ factor in the running time.}

\vspace{-.2cm}
\subsection{Regret bound for Algorithm~\ref{algo:CUCB}}

We provide a gap dependent upper bound on the regret of \textsc{boim-cucb} in Theorem~\ref{thm:CUCBregret}. For this, we define, for $i\in V$, the gap
\vspace{-.2cm}
\[\Delta_{i,\min}\triangleq \min_{S\subset V,~p_i\pa{S;\bw^\star}>0,~ \Delta\pa{S}>0}\Delta\pa{S}.\]
We also define, with $d_k$ being the out-degree of node $k$,
\vspace{-.2cm}
$$p_{i,\max}\triangleq \max_{S\subset V,~p_i(S;\bw^\star)>0}\sum_{k\in V}d_k p_k(S;\bw^\star).$$
\begin{theorem}\label{thm:CUCBregret}
If $\pi$ is the policy described in Algorithm~\ref{algo:CUCB}, then \begin{align*} {R_{B,\varepsilon}}\pa{\pi}\! =\cO\pa{\!\log\!{B}\pa{\sum_{i\in V}\abs{V}\frac{\lambda^\star\!\!+\!d_ip_{i,\max}\abs{V} }{\Delta_{i,\min}} \!}\!\! }\!.\end{align*}
 In addition, if true costs are known, then 
\begin{align*} {R_{B,\varepsilon}}\pa{\pi}\! =\cO\pa{\!\log\!{B}\pa{\sum_{i\in V}\frac{d_ip_{i,\max}\abs{V}^2 }{\Delta_{i,\min}} \!}\!\! }\!.\end{align*}
\end{theorem}
\vspace{-.2cm}

A proof of Theorem~\ref{thm:CUCBregret} can be found in Appendix~\ref{app:CUCBregret}. Notice that the analysis can be easily used for the non budgeted setting. In this case, it reduces to the state-of-the-art analysis of \citet{wang2017improving}, except that we slightly simplify and improve the analysis to replace the factor $\max_{S\subset V}\sum_{k\in V}d_k\II{p_k(S;\bw^\star)>0}$ by a potentially much lower quantity $p_{i,\max}$. In the case this last quantity is still large, we can further improve it by considering slight modifications to the original Algorithm~\ref{algo:CUCB}. This is the purpose of the next section. 

\section{More refined optimistic spreads}
We observe that the factor $p_{i,\max}$ in Theorem~\ref{thm:CUCBregret} can be as large as $\abs{E}$ in the worst case. In other word, if $\Delta=\min_i \Delta_{i,\min}$, the rate can be as large as
$$\cO\pa{\!\log\!{B}\pa{
\frac{\lambda^\star\abs{V}^2\!\!+\!\abs{E}^2\abs{V}^2 }{\Delta}}}\cdot$$
We argue here that we can replace $\abs{E}^2\abs{V}^2$ by $\abs{E}\abs{V}^3\log^2(\abs{V}))$.
Indeed, leveraging on the mutual independence of random variables $W_{ij}$, we can hope to get a tighter confidence region for $\bw^\star$, and thus a provably tighter regret upper bound \citep{magureanu2014,combes2015combinatorial,Degenne2016}. 
We consider the following confidence region from \citet{Degenne2016} (see also \citet{perrault2019exploiting}) and adapted to our setting.
\begin{fact}[Confidence ellipsoid for weights]\label{degenne}
For all $t\geq 2$, with probability at least $1-{1}/\pa{t\log^2\pa{t}},$\[\sum_{ij\in E}\counterw{i}{t-1}\pa{{w_{ij}^\star-\meanw{ij,t-1}}}^2 \leq  \delta\pa{t},\] 
\vspace{-.38cm}

where
$\delta\pa{t}\triangleq 2\log\pa{t} +2\pa{\abs{E}+2}\log\log(t) + {1}.$
\end{fact}
For OIM (both budgeted and non budgeted), there is a large potential gain in the analysis using the confidence region given by Fact~\ref{degenne} compared to simply using an Hoeffding based one, like in \textsc{boim-cucb}. More precisely, for classical combinatorial semi bandits, \citet{Degenne2016} reduced the gap dependent regret upper bound by a factor $\ell/\log^2{\pa{\ell}}$, where in our case $\ell$ can be as large as $\abs{E}$. However, there is also a drawback in practice with such confidence region: computing the optimistic spread might be inefficient, even if an oracle for evaluating the spread is available. Indeed, for a fixed $S\subset V$, the problem of maximizing
\(\bw\mapsto \sigma\pa{S;\bw}\) over $\bw$ belonging to some ellipsoid might be hard, since the objective is not necessarily concave.
We can overcome this issue using the following Fact~\ref{wang} \citep{wen2017online,wang2017improving}.
\begin{fact}[Smoothness property of the spread]\label{wang}
for all $S\subset V,$ and all $\bw,\bw' \in [0,1]^E$,
\vspace{-.2cm}
\begin{align*}\forall k\in V,~\abs{p_k\pa{S;\bw}\!-\!p_k\pa{S;\bw'}}\!\leq\!\! \sum_{ij\in E}\!p_i\pa{S;\bw}\abs{w_{ij}\!-\!w'_{ij}}.\end{align*}
\vspace{-.8cm}

In particular,
\begin{align*}\abs{\sigma\pa{S;\bw}-\sigma\pa{S;\bw'}}\leq \abs{V}\sum_{ij\in E}p_i\pa{S;\bw}\abs{w_{ij}-w'_{ij}}.\end{align*}
\end{fact}

For  $S\subset V$ and $\bw\in \R^E,$ we  define the confidence ``bonus" as follows:
\vspace{-.5cm}

\begin{align*}\bonus\pa{S;\bw}\triangleq \abs{V}\sqrt{\delta\pa{t}\sum_{i\in V,\counterw{i}{t-1}>0}d_i\frac{p_i\pa{S;\bw}^2}{\counterw{i}{t-1}}}
\cdot\end{align*}

Notice, we don't sum on vertices with a zero counter. We compensate this by using the convention $\bar w_{ij,t-1}=1$ when $\counterw{i}{t-1}=0$. We can successively
use Fact~\ref{wang}, Cauchy-Schwartz inequality, and Fact~\ref{degenne} to get, with probability at least $1-{1}/\pa{t\log^2\pa{t}},$ 
\begin{align}\sigma\pa{S;\bw^\star}\leq \sigma\pa{S;\vmeanw{t-1}} + \bonus\pa{S;\vmeanw{t-1}}.\label{UCB}\end{align}
In the same way,  with probability at least $1-{1}/\pa{t\log^2\pa{t}},$ we also have \eqref{UCB*}. 
\begin{align}\sigma\pa{S;\bw^\star}\leq \sigma\pa{S;\vmeanw{t-1}} + \bonus\pa{S;\bw^\star}.\label{UCB*}\end{align}
Contrary to \eqref{UCB}, this ``optimistic spread" can't be used directly by the agent since $\bw^\star$ is not known. 

Although the optimistic spread  defined in \eqref{UCB} is now much easier to compute, there is still a major drawback that remains: As a function of $S\subset V$, $\bonus\pa{S;\vmeanw{t-1}}$ is not necessarily submodular, so the optimistic spread is itself no longer submodular. This is an issue
because submodularity is a crucial property for 
reaching the approximation ratio $ {1-{1}/{e}-\varepsilon}$.
We propose here several submodular upper bound to $\bonus$, defined for $S\subset V$ and $\bw\in \R^E$: 

\begin{itemize}
    \item $\bonus_1$ is actually modular, and simply uses the subadditivity (w.r.t. $S$) of $\bonus$: \[\bonus_1\pa{S;\bw}\!\triangleq\!\abs{V}\!\sum_{j\in S}\!\sqrt{\!\delta\pa{t}\!\!\sum_{i,\counterw{i}{t-1}>0}\!\! d_i\frac{p_i\pa{\sset{j};\bw}^2}{\counterw{i}{t-1}}}\cdot\]
    \item $\bonus_2$ uses the subadditivity of the square root:
    \vspace{-.3cm}
    \[\bonus_2\pa{S;\bw}\triangleq\abs{V}\sum_{i,\counterw{i}{t-1}>0}p_i\pa{S;\bw}\sqrt{\frac{\delta\pa{t}d_i}{\counterw{i}{t-1}}}\cdot\]
    
    \item $\bonus_3$ uses $p_i\pa{S;\bw}^2\leq p_i\pa{S;\bw}$, and is submodular as the composition between a non decreasing concave function (the square root) and a monotone submodular function: \[\bonus_3\pa{S;\bw}\triangleq\abs{V}\sqrt{\delta\pa{t}\sum_{i,\counterw{i}{t-1}>0}d_i\frac{p_i\pa{S;\bw}}{\counterw{i}{t-1}}}\cdot\]
    
    \item $\bonus_4$ uses Jensen's inequality,  and is submodular as the expectation of the square root of a submodular function.
    $$\bonus_4(S;\bw)\triangleq \EE{\abs{V}\sqrt{\sum_{{i\in V, \counterw{i}{t-1}>0,S\reach{\bW} i}}\frac{\delta\pa{t}d_i}{\counterw{i}{t-1}}}},$$
where $\bW\sim\otimes_{ij \in E}\Bernoulli\pa{w_{ij}}$. 
\end{itemize}

We can write the following approximation guarantees for the two first bonus: \begin{align}\bonus\pa{S;\bw}\leq \bonus_1\pa{S;\bw} \!\leq \abs{S}\bonus \pa{S;\bw},\label{bonus1}\end{align}
\[\bonus\pa{S;\bw}\leq \bonus_2\pa{S;\bw} \leq \sqrt{\abs{V}}\bonus \pa{S;\bw}.\]
Notice, another approach to get a submodular bonus is to approximate $p_i\pa{S;\vmeanw{t-1}}$ by the square root of a modular function \citep{goemans2009approximating}. However, not only this bonus would be much more computationally costly to build than ours, but also, we would get only a $\sqrt{\abs{V}}\log\abs{V}$ approximation factor, which is worst than the one with our $\bonus_2$. 
Since increasing the bonus by a factor $\alpha\geq 1$ increases the gap dependent regret upper bound by a factor $\alpha^2$, we only loose a factor $\abs{V}$ for $\bonus_2$, compared to the use of $\bonus$, which is still better than the \textsc{cucb} approach. 
$\bonus_1$ can also be interesting to use when we have some upper bound guarantee on the cardinality of seed sets used (see subsection~\ref{subsec:low}).
An approximation factor for $\bonus_3$ or $\bonus_4$ doesn't seem interesting, because it would involve the
 inverse of triggering probabilities. We can, however, further upper bound $\bonus_3\pa{S;\bw^\star}$ as follows:
\begin{align*}\bonus_3\pa{S;\bw^\star}&= \abs{V}\sqrt{\delta\pa{t}\sum_{i,\counterw{i}{t-1}>0}d_i\frac{p_i\pa{S;\bw^\star}}{\counterw{i}{t-1}}}\\&\leq \abs{V}\sqrt{\sum_{j\in S}\sum_{\substack{i\in V,\\\counterw{i}{t-1}>0}}d_i\frac{\delta\pa{t}p_i\pa{\sset{j};\bw^\star}}{\counterw{i}{t-1}}}
\\&\leq \abs{V}\sqrt{\delta\pa{t}\sum_{j\in S}\abs{E}\pa{\frac{8}{\counterc{j}{t-1}}\wedge 1}}\CommaBin
\end{align*}
where the last inequality only holds under some high probability event, given by the following Proposition~\ref{prop:counter}, involving counters on the costs and counters on the weights.  
 \begin{proposition} Consider the event defined by $\mathfrak{P}_{t}\triangleq\sset{\forall i\in V, {\counterw{i}{t-1}} \geq \delta\pa{t}}$. Then,
 for all $i,j\in V$, 
 \[ \PP{\mathfrak{P}_{t}\text{ and }\frac{\delta\pa{t}p_i\pa{\sset{j};\bw^\star}}{\counterw{i}{t-1}}> \frac{8\delta\pa{t}}{\counterc{j}{t-1}}}\leq 1/t^2.\]
 \label{prop:counter}
 \end{proposition}

 We thus define for $S\subset V$,
 \vspace{-.5cm}
 
 \[\bonus_5\pa{S}\triangleq \abs{V}\sqrt{\delta\pa{t}\sum_{j\in S}\abs{E}\pa{\frac{8}{\counterc{j}{t-1}}\wedge 1}}\cdot\]
 This bonus is much more convenient since it does not depend anymore on $\bw^\star$, and can thus be computed by the agent. Indeed,
 although the first four bonuses are likely to be tighter than this last submodular $\bonus_5$, their dependence in $\bw$ forces us to use them for $\bw=\vmeanw{t-1}$. Even if this doesn't pose any problem in practice, this is more difficult to handle in theory since it would involve optimistic estimates on $p_i(~\cdot~;\vmeanw{t-1})$ itself (see the next section for further details).
 Actually, we will see that the analysis with $\bonus_5$ is slightly better than the one we would get with $\bonus_2$, since we loose a factor $\abs{V}\log^2(\abs{V})/\log^2(\abs{E})$ compared to the use of $\bonus$. In addition, 
 it allows a much more interesting constant term in the regret upper bound thanks to the suppression of the dependence in $\bw^\star$.

Although we can improve the analysis based on Fact~\ref{degenne}~and~\ref{wang},
the inequality in this last fact may be be less rough in practice (we confirm this in section~\ref{sec:exp}). In that case, we suffer from this roughness, since we actually use Fact~\ref{wang} to design our bonus. In contrast, \textsc{boim-cucb} only uses it in the analysis, and can therefore adapt to a better smoothness inequality. 
Thus,   
we consider \textsc{boim-cucb}$_5$, where
we first compute $S_t$ using the \textsc{boim-cucb} approach, and accept it only if        
\vspace{-.7cm}

\begin{align}\sigma\pa{S_t;\bw_t} \leq \sigma\pa{S_t;\vmeanw{t-1}} + \bonus_5\pa{S_t},\label{condition}\end{align}
\vspace{-.6cm}

otherwise, we chose $S_t$ maximizing $\sigma\pa{S;\vmeanw{t-1}} + \bonus_5\pa{S}$. For technical reason (due to Proposition~\ref{prop:counter}), we replace $S_t$ by $S_t\cup\sset{j}$ if it exists a $ j\in V$, such that \begin{align} \counterw{j}{t-1}<\delta\pa{t}\label{cond:counter}.\end{align}

 We thus both enjoy the theoretical advantages of $\bonus_5$ and the practical advantages of \textsc{boim-cucb}. We give the following regret bounds for this approach. 
 
 \begin{theorem}\label{thm:regret}
If $\pi$ is the policy following \textsc{boim-cucb}$_5$, then \begin{align*} {R_{B,\varepsilon}}\pa{\pi}\! =\!\cO\!\pa{\!\log\!{B}\pa{\sum_{i\in V}\abs{V}\frac{\lambda^\star\!\!+\!\abs{V}\abs{E} \log^2~\!\!~\!\!{\abs{V}}}{\Delta_{i,\min}} \!+\!\!\lambda^\star\abs{V}^2\!}\!\! }\!.\end{align*}
\end{theorem}
\vspace{-.2cm}

A proof of Theorem~\ref{thm:regret} can be found in Appendix~\ref{app:regret}. Such analysis also holds in the non-budgeted setting, and maximizing the spread only instead of the ratio, we can build a policy $\pi$ satisfying the following (with the standard definition of the non-budgeted gaps):
\[R_{T,\varepsilon}\pa{\pi}=\cO\pa{\log{T}\sum_{i\in V}\frac{\abs{V}^2\abs{E}\log^2{\abs{V}}}{\Delta_{i,\min}}}.\]
 The regret rate is thus better than the one from \citet{wang2017improving}, gaining a factor $\abs{E}/\pa{\abs{V}\log^2\pa{\abs{V}}}$. 
 
 \section{Improvements using $\bonus_1$ and $\bonus_4$}
 
 In this section, we show that the use of $\bonus_1$ and $\bonus_4$ leads to a better regret leading term, at the cost of a large second order term. In the following, we propose \textsc{boim-cucb}$_1$ (resp. \textsc{boim-cucb}$_4$), that are the same approach as \textsc{boim-cucb}$_5$ with $\bonus_1\pa{~\cdot~;\vmeanw{t-1}}$ (resp. $\bonus_4\pa{~\cdot~;\vmeanw{t-1}}$) instead of $\bonus_5$, and where condition~\eqref{cond:counter} is replaced by
\[\exists j\in V, \counterw{j}{t-1}\leq  \abs{E}\delta\pa{t}.\]  
 
\subsection{$\bonus_1$ for low cardinality seed sets}\label{subsec:low}
In many real world scenarios, maximal cardinality of seed set is small compared to $\abs{V}$. Indeed, in the non-budgeted setting, it is limited by $m$, and it is usually assumed that $m$ is much smaller than $\abs{V}$. In the budgeted setting, we will see in section~\ref{sec:budget} how to limit the cost of the chosen seeds, and this is likely to also induce a limit on the cardinality of seeds. 
Using $\bonus_1$ is more appropriate in this situation, according to the approximation factor \eqref{bonus1}. 
We state in Theorem~\ref{thm:low} the regret bound for \textsc{boim-cucb}$_1$.
\begin{theorem}
\label{thm:low}
If $\pi$ is the policy \textsc{boim-cucb}$_1$,
and if all seeds selected have a cardinality bounded by $m,$
then we have \begin{align*}{R_{B,\varepsilon}}\pa{\pi}=\cO\left( \log{B}\left(\sum_{i\in V}m\frac{\lambda^\star +m\abs{V}^2d_i\log^2\pa{\abs{E}}}{\Delta_{i,\min}}\right.\right.\\\left.\left.+\lambda^\star{\abs{V}^2}\abs{E}\sizecorr{\sum_{i\in V}\frac{\lambda^\star K+\min\sset{\abs{V},K^2}\abs{V}d_i\log^2\pa{\abs{E}}}{\Delta_{i,\min}}}\right)\right).\end{align*}
\end{theorem}
\vspace{-.4cm}

A proof can be found in Appendix~\ref{app:low}. As previously, we can state the following non-budgeted version, with seed set cardinality constrained by $m$:
\begin{align*}{R_{T,\varepsilon}}\pa{\pi}\!=\!\cO\!\left(\! \log{T}\left(\sum_{i\in V}\!\frac{m^2\abs{V}^2d_i\log^2\pa{\abs{E}}}{\Delta_{i,\min}}\!+\!\abs{E}\abs{V}^2\sizecorr{\sum_{i\in V}\frac{\lambda^\star K+\min\sset{\abs{V},K^2}\abs{V}d_i\log^2\pa{\abs{E}}}{\Delta_{i,\min}}}\!\right)\!\!\right)\!.\end{align*}
\vspace{-.2cm}

Notice, for both settings, there is an improvement in the main term (the gap dependent one), in the case $m\leq \sqrt{\abs{V}}$. However, there is also a higher gap independent term that appears. 

\subsection{$\bonus_4$: the same performance as $\bonus$? }

We show here that the regret with $\bonus_4$ is of the same order as what we would have had with $\bonus$ (which is not submodular). However, $\bonus_4$ does not have the calculation guarantees of the other bonuses.
We state in Theorem~\ref{thm:jensen} the regret bound for the policy \textsc{boim-cucb}$_4$. Notice that we obtain a bound whose leading term improves by a factor $\abs{E}/\log^2\abs{E}$ that of \textsc{boim-cucb}.
\begin{theorem}
\label{thm:jensen}
If $\pi$ is the policy \textsc{boim-cucb}$_4$,
then we have \begin{align*}{R_{B,\varepsilon}}\pa{\pi}=\cO\left( \log{B}\left(\sum_{i\in V}\frac{\abs{V}\lambda^\star +\abs{V}^2d_i\log^2\pa{\abs{E}}}{\Delta_{i,\min}}\right.\right.\\\left.\left.+\lambda^\star{\abs{V}^2}\abs{E}\sizecorr{\sum_{i\in V}\frac{\lambda^\star K+\min\sset{\abs{V},K^2}\abs{V}d_i\log^2\pa{\abs{E}}}{\Delta_{i,\min}}}\right)\right).\end{align*}
\end{theorem}
\vspace{-.45cm}
The proof is given in in Appendix~\ref{app:jensen}. As previously, we can state the following non-budgeted version. Notice that the cardinality constrain does not appear in the bound.
\begin{align*}{R_{T,\varepsilon}}\pa{\pi}\!=\!\cO\!\left(\! \log{T}\left(\sum_{i\in V}\!\frac{\abs{V}^2d_i\log^2\pa{\abs{E}}}{\Delta_{i,\min}}\!+\!\abs{E}\abs{V}^2\sizecorr{\sum_{i\in V}\frac{\lambda^\star K+\min\sset{\abs{V},K^2}\abs{V}d_i\log^2\pa{\abs{E}}}{\Delta_{i,\min}}}\!\right)\!\!\right)\!.\end{align*}
\vspace{-.45cm}
  
In spite of the superiority in terms of regret of the use of $\bonus_4$, we must point out that, in the worst case, the calculation of this bonus may require a number of sample (and thus a time complexity) polynomial in $t$, which does not meet the criterion of efficiency that we set ourselves at the beginning of the paper.
 \vspace{-.3cm}
\section{Knapsack constraint for known costs}\label{sec:budget}

In their setting, \citet{wang2020fast} considered the relaxed constraint 
\vspace{-.1cm}
\begin{align}\EE{\be_{S\cup\sset{0}}\bc^\star}\leq b,\label{rel:knap_cons}\end{align} instead of ratio maximization, where the expectation is over the possible randomness of $S$. When true costs are known to the agent, we can actually combine the two settings: a seed set $S$ can be chosen only if it satisfies \eqref{rel:knap_cons}. In this section, we describe modifications this new setting implies. First of all, the regret definition is impacted, and $F_B^\star$ is now maximal for policies respecting the constraint \eqref{rel:knap_cons} within each round.
Naturally, the definitions of $\lambda^\star$ and $S^\star$ are also modified accordingly.
Otherwise, apart from Algorithm~\ref{algo:greedy}, there is conceptually no change in the approaches that have been described in this paper. We now described the modification needed to make Algorithm~\ref{algo:greedy} works in this setting. The same sequence of set $S_k$ is considered, but instead of choosing the set that maximizes the ratio over all $k\in \sset{0,\dots,\abs{V}}$, we restrict the maximization to $k\in \sset{0,\dots,j}$, where $j$ is the first index such that
$\be_{S_{j}\cup\sset{0}}\transpose \bc
^\star> b$. If this maximizer is not $S_j$, then it satisfies the constraint and is output. Else, we output $S_j$ with probability $(b-\be_{S_{j-1}\cup\sset{0}}\transpose \bc^\star)/c^\star_{j}$ and $S_{j-1} $ with probability $1-(b-\be_{S_{j-1}\cup\sset{0}}\transpose \bc^\star)/c^\star_{j}$. This way, the expected cost of the output is $b$. We prove in Appendix~\ref{app:greedy_bis} the following Proposition~\ref{prop:greedy_bis}, giving an approximation factor of $1-1/e$ for the above modification of Algorithm~\ref{algo:greedy}. 

\begin{proposition}
\label{prop:greedy_bis}
The solution $S$ obtained by the modified Algorithm~\ref{algo:greedy} is such that:
\vspace{-.1cm}
\[\pa{1-e^{-1}}\frac{\EE{\sigma\pa{S^\star}}}{\EE{\be_{S^\star\cup\sset{0}}\transpose \bc^\star}}\leq \frac{\EE{\sigma\pa{S}}}{\EE{\be_{S\cup\sset{0}}\transpose \bc^\star}}\CommaBin\]
where the expectation is over the possible randomness of $S,S^\star$.
\end{proposition}

\vspace{-.3cm}
\section{Experiments}
\vspace{-.1cm}
\label{sec:exp}
\begin{figure}[h]
\begin{center}
\includegraphics[width=\columnwidth]{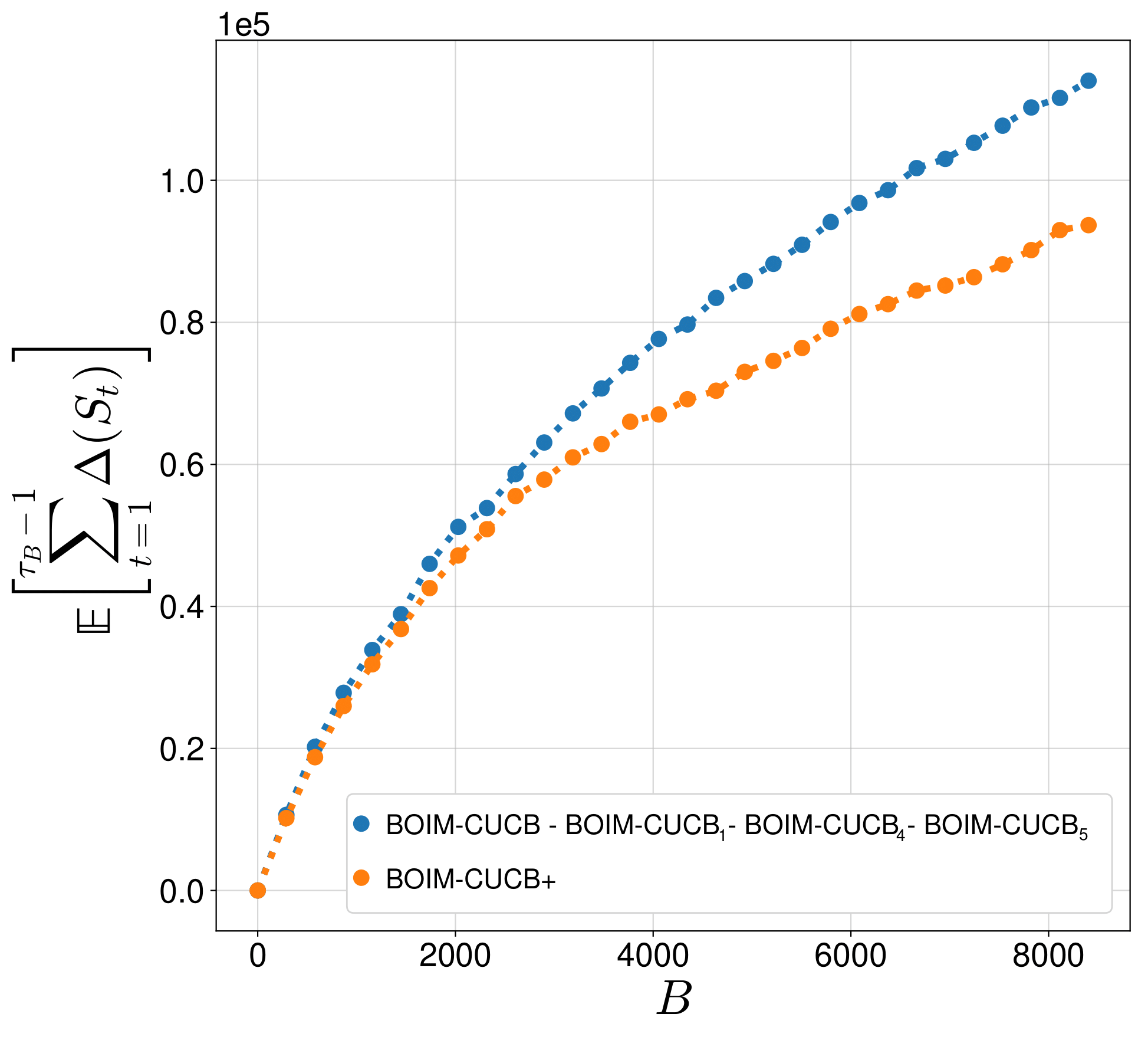}
\vspace{-.8cm}
\caption{Regret curves with respect to the budget $B$ (expectation computed by averaging over 10 independent simulations).}
\label{exp:regret}
\end{center}
\vspace{-.2cm}
\end{figure}

In this section, we present an experiment for Budgeted OIM. In Figure~\ref{exp:regret}, we plot  $\EE{\sum_{t=1}^{\tau_B-1} \Delta\pa{S_t}{}}$ with respect to the budget $B$ used, running over up to $T=10000$ rounds. This quantity is a good approximations to the true regret according to Proposition~\ref{prop:budgeted_regret}. Plotting the true regret would require to compute $F_B^\star$, which is NP-Hard to do.  
We consider a  subgraph of Facebook network \citep{snapnets}, with $\abs{V}=333$ and $\abs{E}=5038$, as in \citet{wen2017online}. We take $\bw^\star\sim U\pa{0,0.1}^{\otimes E}$ and take deterministic, known costs with $c^\star_0=1$, and $c^\star_i=d_i/\max_{j\in V} d_j$. 
\textsc{{boim-cucb}+} is the same approach as \textsc{boim-cucb}$_5$ with $\bonus\pa{~\cdot~;\vmeanw{t-1}}$ instead of $\bonus_5$, ignoring that $\bonus\pa{~\cdot~;\vmeanw{t-1}}$ is not submodular (it is only sub-additive).

We observed that in \textsc{boim-cucb}$_1$, \textsc{boim-cucb}$_4$, \textsc{boim-cucb}$_5$, 
Condition~\ref{condition} (with the correct bonus instead of $\bonus_5$)  always holds, meaning that those algorithms coincide with  
\textsc{{boim-cucb}} \emph{in practice}, and that the gain only appears through the analysis.
We thus plot a single curve for these 4 algorithms in Figure~\ref{exp:regret}. On the other hand, we observe only a slight gain of \textsc{{boim-cucb}+} compared to \textsc{{boim-cucb}}.

Our experiments confirm that Fact~\ref{wang} is less rough in practice, as we already anticipated. Indeed, our submodular bonuses 
are not tight enough to compete with \textsc{boim-cucb}, although we gain in the analysis. The slight gain that we have for $\textsc{{boim-cucb}+}$ suggests that the issue is not only about the tightness of a submodular upper bound, but rather about the tightness of Fact~\ref{wang}. This is supported by the following observation we made: for the Facebook subnetwork, for $1000$ random draws of a seed set and vector pairs in $[0,0.1]
^E$, the ratio of the RHS and the LHS in Fact~\ref{wang} is each time greater than $0.4\abs{V}$.  

In Appendix~\ref{app:exp}, we conducted further experiments on a synthetic graph comparing \textsc{boim-cucb} to \textsc{boim-{cucb}-regularized}, which greedily maximizes the regularized spread
$S\mapsto\sigma(S;\bw_t)-\lambda \be_S\bc_t,$
where $\lambda$ is a parameter to set.
We observed that for an appropriate choice of $\lambda$, a performance similar to \textsc{boim-cucb} can be obtained.

\vspace{-.2cm}
\section{Discussion and Future work}

We introduced a new Budgeted OIM problem, taking both the costs of influencers and fixed costs into account in the seed selection, instead of the usual cardinality constraint. This better represents the current challenges in viral marketing, since top influencers tend to be more and more costly. Our fixed cost can also be seen as the time that a round takes: A null fixed cost would mean that reloading the network to get a new independent instance is \emph{free} and \emph{instantaneous}.\footnote{In this case, using $\abs{S}$ rounds choosing each time a single different influencer $i\in S$ is better than choosing the whole $S$ in a single round.} Obviously, this is not realistic. 
We also provided an algorithm for Budgeted OIM under the IC model and the edge level semi-bandit feedback setting. 

Interesting future directions of research would be to explore other kinds of feedback or diffusion models for Budgeted OIM.
For practical scalability, it would also be good to investigate the incorporation of
the linear generalization framework \citep{wen2017online} into Budgeted OIM. Notice, this extension is \emph{not} straightforward if we want to keep our tighter confidence region. 
More precisely, we believe that a \emph{linear semi-bandit} approach that is aware of independence between edge observations should be developed (the linear generalization approach of \citet{wen2017online} treats each edge observation as \emph{arbitrary correlated}). 

In addition to this,  
exploring how the use of Fact~\ref{wang} \emph{in the Algorithm} might be avoided while still using confidence region given by Fact~\ref{degenne} would surely improve the algorithms. One possible way would be to use a Thompson Sampling (TS) approach \citep{wang2018,perrault2020statistical}, where the prior takes into account the mutual independence of weights. However, \citet{wang2018} proved in their Theorem~2 that TS gives linear approximation regret for some special approximation algorithms.
Thus, we would have to use some specific property of the \textsc{greedy} approximation algorithm we use.   

 \newpage

\section*{Acknowledgements} The research presented was supported by European CHIST-ERA project DELTA, French Ministry of Higher Education and Research, Nord-Pas-de-Calais Regional Council,  French National Research Agency project BOLD (ANR19-CE23-0026-04).

\bibliography{lib}

\begin{thebibliography}{}

\bibitem[Auer et~al., 2002]{auer2002finite}
Auer, P., Cesa-Bianchi, N., and Fischer, P. (2002).
\newblock {Finite-time analysis of the multiarmed bandit problem}.
\newblock {\em Machine Learning}, 47(2-3):235--256.

\bibitem[Bai et~al., 2016]{bai2016algorithms}
Bai, W., Iyer, R., Wei, K., and Bilmes, J. (2016).
\newblock Algorithms for optimizing the ratio of submodular functions.
\newblock In {\em International Conference on Machine Learning}, pages
  2751--2759.

\bibitem[Carpentier and Valko, 2016]{carpentier2016revealing}
Carpentier, A. and Valko, M. (2016).
\newblock {Revealing graph bandits for maximizing local influence}.
\newblock In {\em International Conference on Artificial Intelligence and
  Statistics}.

\bibitem[Chen et~al., 2010]{chen2010scalable}
Chen, W., Wang, C., and Wang, Y. (2010).
\newblock {Scalable influence maximization for prevalent viral marketing in
  large-scale social networks}.
\newblock In {\em Knowledge Discovery and Data Mining}.

\bibitem[Chen et~al., 2013]{chen13a}
Chen, W., Wang, Y., and Yuan, Y. (2013).
\newblock Combinatorial multi-armed bandit: General framework and applications.
\newblock In {\em Proceedings of the 30th International Conference on Machine
  Learning}, volume~28 of {\em Proceedings of Machine Learning Research}, pages
  151--159, Atlanta, Georgia, USA. PMLR.

\bibitem[Chen et~al., 2016]{Chen2015combinatorial}
Chen, W., Wang, Y., and Yuan, Y. (2016).
\newblock {Combinatorial multi-armed bandit and its extension to
  probabilistically triggered arms}.
\newblock {\em Journal of Machine Learning Research}, 17.

\bibitem[Cohen, 1997]{cohen1997size}
Cohen, E. (1997).
\newblock Size-estimation framework with applications to transitive closure and
  reachability.
\newblock {\em Journal of Computer and System Sciences}, 55(3):441--453.

\bibitem[Cohen, 2016]{cohen2016minhash}
Cohen, E. (2016).
\newblock Minhash sketches.

\bibitem[Cohen et~al., 2014]{cohen2014sketch}
Cohen, E., Delling, D., Pajor, T., and Werneck, R.~F. (2014).
\newblock Sketch-based influence maximization and computation: Scaling up with
  guarantees.
\newblock In {\em Proceedings of the 23rd ACM International Conference on
  Conference on Information and Knowledge Management}, pages 629--638. ACM.

\bibitem[Cohen and Kaplan, 2007]{cohen2007summarizing}
Cohen, E. and Kaplan, H. (2007).
\newblock Summarizing data using bottom-k sketches.
\newblock In {\em Proceedings of the twenty-sixth annual ACM symposium on
  Principles of distributed computing}, pages 225--234. ACM.

\bibitem[Combes et~al., 2015]{combes2015combinatorial}
Combes, R., Shahi, M. S. T.~M., Proutiere, A., and Others (2015).
\newblock {Combinatorial bandits revisited}.
\newblock In {\em Advances in Neural Information Processing Systems}, pages
  2116--2124.

\bibitem[Degenne and Perchet, 2016]{Degenne2016}
Degenne, R. and Perchet, V. (2016).
\newblock Combinatorial semi-bandit with known covariance.
\newblock In {\em Advances in Neural Information Processing Systems 29}, pages
  2972--2980. Curran Associates, Inc.

\bibitem[Ding et~al., 2013]{Ding2013}
Ding, W., Qiny, T., Zhang, X.-D., and Liu, T.-Y. (2013).
\newblock {Multi-armed bandit with budget constraint and variable costs}.

\bibitem[Feige, 1998]{feige1998threshold}
Feige, U. (1998).
\newblock {A threshold of $\ln n$ for approximating set cover}.
\newblock {\em Journal of the ACM (JACM)}, 45(4):634--652.

\bibitem[Fujishige, 2005]{fujishige2005submodular}
Fujishige, S. (2005).
\newblock {\em {Submodular functions and optimization}}.
\newblock Annals of discrete mathematics.

\bibitem[Goemans et~al., 2009]{goemans2009approximating}
Goemans, M.~X., Harvey, N.~J., Iwata, S., and Mirrokni, V. (2009).
\newblock Approximating submodular functions everywhere.
\newblock In {\em Proceedings of the twentieth annual ACM-SIAM symposium on
  Discrete algorithms}, pages 535--544. Society for Industrial and Applied
  Mathematics.

\bibitem[Gomez-Rodriguez et~al., 2012]{Gomez-Rodriguez2012}
Gomez-Rodriguez, M., Leskovec, J., and Krause, A. (2012).
\newblock Inferring networks of diffusion and influence.
\newblock {\em ACM Trans. Knowl. Discov. Data}, 5(4):21:1--21:37.

\bibitem[Goyal et~al., 2010]{Goyal2010}
Goyal, A., Bonchi, F., and Lakshmanan, L.~V. (2010).
\newblock Learning influence probabilities in social networks.
\newblock In {\em Proceedings of the Third ACM International Conference on Web
  Search and Data Mining}, WSDM '10, pages 241--250, New York, NY, USA. ACM.

\bibitem[{Goyal} et~al., 2011]{6137225}
{Goyal}, A., {Lu}, W., and {Lakshmanan}, L. V.~S. (2011).
\newblock Simpath: An efficient algorithm for influence maximization under the
  linear threshold model.
\newblock In {\em 2011 IEEE 11th International Conference on Data Mining},
  pages 211--220.

\bibitem[Kakade et~al., 2009]{kakade2009playing}
Kakade, S.~M., Kalai, A.~T., and Ligett, K. (2009).
\newblock Playing games with approximation algorithms.
\newblock {\em SIAM Journal on Computing}, 39(3):1088--1106.

\bibitem[Kempe et~al., 2003]{kempe2003maximizing}
Kempe, D., Kleinberg, J., and Tardos, {\'E}. (2003).
\newblock Maximizing the spread of influence through a social network.
\newblock In {\em Proceedings of the ninth ACM SIGKDD international conference
  on Knowledge discovery and data mining}, pages 137--146. ACM.

\bibitem[Khuller et~al., 1999]{khuller1999budgeted}
Khuller, S., Moss, A., and Naor, J.~S. (1999).
\newblock The budgeted maximum coverage problem.
\newblock {\em Information processing letters}, 70(1):39--45.

\bibitem[Krause and Guestrin, 2005]{krause2005note}
Krause, A. and Guestrin, C. (2005).
\newblock {A Note on the Budgeted Maximization of Submodular Functions}.
\newblock Technical Report June, CMU.

\bibitem[Leskovec et~al., 2007]{Leskovec2007}
Leskovec, J., Krause, A., Guestrin, C., Faloutsos, C., Faloutsos, C.,
  VanBriesen, J., and Glance, N. (2007).
\newblock Cost-effective outbreak detection in networks.
\newblock In {\em Proceedings of the 13th ACM SIGKDD International Conference
  on Knowledge Discovery and Data Mining}, KDD '07, pages 420--429, New York,
  NY, USA. ACM.

\bibitem[Leskovec and Krevl, 2014]{snapnets}
Leskovec, J. and Krevl, A. (2014).
\newblock {SNAP Datasets}: {Stanford} large network dataset collection.
\newblock \url{http://snap.stanford.edu/data}.

\bibitem[Lugosi et~al., 2019]{lugosi19a}
Lugosi, G., Neu, G., and Olkhovskaya, J. (2019).
\newblock Online influence maximization with local observations.
\newblock In {\em Proceedings of the 30th International Conference on
  Algorithmic Learning Theory}, volume~98 of {\em Proceedings of Machine
  Learning Research}, pages 557--580, Chicago, Illinois. PMLR.

\bibitem[Magureanu et~al., 2014]{magureanu2014}
Magureanu, S., Combes, R., and Proutiere, A. (2014).
\newblock Lipschitz bandits: Regret lower bound and optimal algorithms.
\newblock In {\em Proceedings of The 27th Conference on Learning Theory},
  volume~35 of {\em Proceedings of Machine Learning Research}, pages 975--999,
  Barcelona, Spain. PMLR.

\bibitem[Minoux, 1978]{minoux1978accelerated}
Minoux, M. (1978).
\newblock {Accelerated greedy algorithms for maximizing submodular set
  functions}.
\newblock {\em Optimization Techniques}, pages 234--243.

\bibitem[Mitzenmacher and Upfal, 2017]{mitzenmacher2017probability}
Mitzenmacher, M. and Upfal, E. (2017).
\newblock {\em Probability and computing: randomization and probabilistic
  techniques in algorithms and data analysis}.
\newblock Cambridge university press.

\bibitem[Nemhauser et~al., 1978]{Nemhauser1978}
Nemhauser, G.~L., Wolsey, L.~A., and Fisher, M.~L. (1978).
\newblock {An analysis of approximations for maximizing submodular set
  functions-I}.
\newblock {\em Mathematical Programming}, 14(1):265--294.

\bibitem[Netrapalli and Sanghavi, 2012]{Netrapalli2012}
Netrapalli, P. and Sanghavi, S. (2012).
\newblock Learning the graph of epidemic cascades.
\newblock {\em SIGMETRICS Perform. Eval. Rev.}, 40(1):211--222.

\bibitem[Nguyen and Zheng, 2013]{nguyen2013budgeted}
Nguyen, H. and Zheng, R. (2013).
\newblock On budgeted influence maximization in social networks.
\newblock {\em IEEE Journal on Selected Areas in Communications},
  31(6):1084--1094.

\bibitem[Perrault et~al., 2020a]{perrault2020statistical}
Perrault, P., Boursier, E., Perchet, V., and Valko, M. (2020a).
\newblock Statistical efficiency of thompson sampling for combinatorial
  semi-bandits.
\newblock {\em arXiv preprint arXiv:2006.06613}.

\bibitem[Perrault et~al., 2019a]{perrault2019exploiting}
Perrault, P., Perchet, V., and Valko, M. (2019a).
\newblock Exploiting structure of uncertainty for efficient matroid
  semi-bandits.
\newblock In {\em Proceedings of the 36th International Conference on Machine
  Learning}, volume~97 of {\em Proceedings of Machine Learning Research}, pages
  5123--5132, Long Beach, California, USA. PMLR.

\bibitem[Perrault et~al., 2019b]{pmlr-v89-perrault19a}
Perrault, P., Perchet, V., and Valko, M. (2019b).
\newblock Finding the bandit in a graph: Sequential search-and-stop.
\newblock In {\em Proceedings of Machine Learning Research}, volume~89 of {\em
  Proceedings of Machine Learning Research}, pages 1668--1677. PMLR.

\bibitem[Perrault et~al., 2020b]{perrault2020covariance}
Perrault, P., Perchet, V., and Valko, M. (2020b).
\newblock Covariance-adapting algorithm for semi-bandits with application to
  sparse outcomes.
\newblock {\em Proceedings of Machine Learning Research}, 125:1--33.

\bibitem[Saito et~al., 2008]{Saito2008}
Saito, K., Nakano, R., and Kimura, M. (2008).
\newblock Prediction of information diffusion probabilities for independent
  cascade model.
\newblock In {\em Knowledge-Based Intelligent Information and Engineering
  Systems}, pages 67--75, Berlin, Heidelberg. Springer Berlin Heidelberg.

\bibitem[Streeter and Golovin, 2009]{streeter2009online}
Streeter, M. and Golovin, D. (2009).
\newblock An online algorithm for maximizing submodular functions.
\newblock In {\em Advances in Neural Information Processing Systems}, pages
  1577--1584.

\bibitem[Tang et~al., 2015]{Tang2015}
Tang, Y., Shi, Y., and Xiao, X. (2015).
\newblock Influence maximization in near-linear time: A martingale approach.
\newblock In {\em Proceedings of the 2015 ACM SIGMOD International Conference
  on Management of Data}, SIGMOD '15, pages 1539--1554, New York, NY, USA. ACM.

\bibitem[Tang et~al., 2014]{tang2014influence}
Tang, Y., Xiao, X., and Shi, Y. (2014).
\newblock Influence maximization: Near-optimal time complexity meets practical
  efficiency.
\newblock In {\em Proceedings of the 2014 ACM SIGMOD international conference
  on Management of data}, pages 75--86. ACM.

\bibitem[Vaswani et~al., 2017]{vaswani2017model}
Vaswani, S., Kveton, B., Wen, Z., Ghavamzadeh, M., Lakshmanan, L.~V., and
  Schmidt, M. (2017).
\newblock Model-independent online learning for influence maximization.
\newblock In {\em Proceedings of the 34th International Conference on Machine
  Learning-Volume 70}, pages 3530--3539. JMLR. org.

\bibitem[Vaswani et~al., 2015]{vaswani2015influence}
Vaswani, S., Lakshmanan, L. V.~S., and {Mark Schmidt} (2015).
\newblock {Influence maximization with bandits}.
\newblock In {\em NIPS workshop on Networks in the Social and Information
  Sciences 2015}.

\bibitem[Wang and Chen, 2017]{wang2017improving}
Wang, Q. and Chen, W. (2017).
\newblock {Improving regret bounds for combinatorial semi-bandits with
  probabilistically triggered arms and its applications}.
\newblock In {\em Neural Information Processing Systems}.

\bibitem[Wang and Chen, 2018]{wang2018}
Wang, S. and Chen, W. (2018).
\newblock {Thompson Sampling for Combinatorial Semi-Bandits}.

\bibitem[Wang et~al., 2020]{wang2020fast}
Wang, S., Yang, S., Xu, Z., and Truong, V.-A. (2020).
\newblock Fast thompson sampling algorithm with cumulative oversampling:
  Application to budgeted influence maximization.
\newblock {\em arXiv preprint arXiv:2004.11963}.

\bibitem[Wen et~al., 2017]{wen2017online}
Wen, Z., Kveton, B., Valko, M., and Vaswani, S. (2017).
\newblock {Online influence maximization under independent cascade model with
  semi-bandit feedback}.
\newblock In {\em Neural Information Processing Systems}.

\bibitem[Xia et~al., 2016]{xia2016budgeted}
Xia, Y., Qin, T., Ma, W., Yu, N., and Liu, T.-Y. (2016).
\newblock {Budgeted multi-armed bandits with multiple plays}.
\newblock In {\em Proceedings of the Twenty-Fifth International Joint
  Conference on Artificial Intelligence}, pages 2210--2216. AAAI Press.

\end{thebibliography}

\clearpage
\appendix
\onecolumn

\section{Proof of Proposition~\ref{prop:budgeted_regret}}\label{app:budgeted_regret}\begin{proof}Let $\alpha={1-{1}/{e}-\varepsilon}$.
In the proof, 
we shall consider several policies $\pi$ one after the other. In each case, we will denote by $S_t$ the seed selected by $\pi$ at round $t$, and $\tau_B$ the random round where $\pi$ has exhausted its budget. We denote $\pa{\cH_{t}}_{t\geq 1}$ the filtration corresponding to $\pa{\bW_t,\bC_t}_{t\geq 1}$. Recall that $S_t$ and $B_{t-1}$ are both measurable with respect to $\cH_{t-1}$. 

Consider first the policy $\pi$ that selects $S_t=S^\star\in \argmax_{S\subset V}{\EE{\sigma\pa{S;\bW}}}{\EE{\be_{S\cup\sset{0}}\transpose \bC}}^{-1}\!\!$ at each round $t\geq 1$. We can write
\begin{align*}
    F_{B}^\star + \abs{V}
    &\geq
    F_{B}^\star + {\EE{\sigma\pa{S^\star;\bW}}}
    \\&\geq
    F_{B}\pa{\pi} + {\EE{\sigma\pa{S^\star;\bW}}} & \text{definition of }F_{B}^\star
    \\&=
     \sum_{t\geq 1} \EE{\sigma\pa{S^\star;\bW_t}\II{B_{t-1}\geq 0}}
     \\&=
     \sum_{t\geq 1} \EE{{\EE{\sigma\pa{S^\star;\bW}}}\II{B_{t-1}\geq 0}} & \text{conditioning on }\cH_{t-1}
     \\&=\lambda^\star
     \sum_{t\geq 1} \EE{{\EE{\be_{S^\star\cup\sset{0}}\transpose \bC}}\II{B_{t-1}\geq 0}}
     \\&= \lambda^\star
     \sum_{t\geq 1} \EE{\pa{\Fixed{t} + \be_{S^\star}\transpose\bC_t}\II{B_{t-1}\geq 0}} & \text{conditioning on }\cH_{t-1}
     \\&\geq  \lambda^\star B & \text{definition of }\tau_B.
\end{align*}
We can use the inequality \begin{align}F_{B}^\star + \abs{V}\geq \lambda^\star B \label{F_B*}\end{align}
with any policy
 $\pi$ in the following ways: 
 \begin{itemize}
     \item First, we can bound the cost part in the cumulative gap: \begin{align*}\alpha\lambda^\star\EE{\sum_{t=1}^{\tau_B-1} {\EE{\be_{S_t\cup\sset{0}}\transpose \bC}}} &\leq
     \alpha\lambda^\star\sum_{t\geq 1} \EE{{\EE{\be_{S_t\cup\sset{0}}\transpose \bC}}\II{B_{t-1}\geq 0}} & B_{t}\geq 0 \imp B_{t-1}\geq 0
      \\&= \alpha\lambda^\star\sum_{t\geq 1} \EE{\pa{\be_{S_t}\transpose \bC_t+\Fixed{t}}\II{B_{t-1}\geq 0}} & \text{conditioning on }\cH_{t-1}
      \\&\leq \alpha\lambda^\star B+\alpha\lambda^\star\pa{1 +\abs{V}} &\text{definition of }\tau_B,
      \\&\leq \alpha F_{B}^\star + \alpha\abs{V} +\alpha\lambda^\star\pa{1 +\abs{V}} &\text{inequality }\eqref{F_B*}.
 \end{align*}
\item Next, we can bound the reward part:
\begin{align*}\EE{\sum_{t=1}^{\tau_B-1}{\EE{\sigma\pa{S_t;\bW}}} }&\geq
    \EE{\sum_{t=1}^{\tau_B}{\EE{\sigma\pa{S_t;\bW}}} } -\abs{V}\\&=  \sum_{t\geq 1} \EE{{\EE{\sigma\pa{S_t;\bW}}}\II{B_{t-1}\geq 0}}-\abs{V}
    \\&= 
    \sum_{t\geq 1} \EE{\sigma\pa{S_t;\bW_t}\II{B_{t-1}\geq 0}}-\abs{V}
    & \text{conditioning on }\cH_{t-1}
    \\&\geq 
    \sum_{t\geq 1} \EE{\sigma\pa{S_t;\bW_t}\II{B_{t}\geq 0}}-\abs{V} & B_{t}\geq 0 \imp B_{t-1}\geq 0
    \\&=
    F_B\pa{\pi}-\abs{V}.
\end{align*} \end{itemize}
Adding these two inequalities, we get the following upper bound on the cumulative gap:
\[\EE{\sum_{t=1}^{\tau_B-1} \Delta\pa{S_t}{}} = \alpha\lambda^\star\EE{\sum_{t=1}^{\tau_B-1} {\EE{\be_{S_t\cup\sset{0}}\transpose \bC}}} - \EE{\sum_{t=1}^{\tau_B-1}{\EE{\sigma\pa{S_t;\bW}}} } \leq R_{B,\varepsilon}\pa{\pi} + (\alpha+1)\abs{V} +\alpha\lambda^\star\pa{1 +\abs{V}}.\]

In the same way, we can derive a lower bound on $\EE{\sum_{t=1}^{\tau_B-1} \Delta\pa{S_t}{}}$, considering first the policy $\pi$ such that $F_{B}(\pi)=F_B^\star$: 

\begin{align*}
     F_B^\star&=
     \sum_{t\geq 1} \EE{\sigma\pa{S_t;\bW_t}\II{B_t\geq 0}}
    \\&\leq 
     \sum_{t\geq 1} \EE{\sigma\pa{S_t;\bW_t}\II{B_{t-1}\geq 0}} & B_{t}\geq 0 \imp B_{t-1}\geq 0
    \\&=
     \sum_{t\geq 1} \EE{{\EE{\sigma\pa{S_t;\bW}}}\II{B_{t-1}\geq 0}} & \text{conditioning on }\cH_{t-1}
    \\&\leq
     \sum_{t\geq 1} \EE{{\lambda^\star}{\EE{\be_{S_t\cup\sset{0}}\transpose \bC}}\II{B_{t-1}\geq 0}} & \text{definition of }\lambda^\star
    \\&=
      {\lambda^\star}\EE{\sum_{t=1}^{\tau_B}\pa{\Fixed{t} + \be_{S_t}\transpose\bC_t}} & \text{conditioning on }\cH_{t-1}
     \\&\leq
      {\lambda^\star}\pa{B+1 +\abs{V}} & \text{definition of }\tau_B.
\end{align*}
i.e., \begin{align} F_{B}^\star -\lambda^\star\pa{1 +\abs{V}} \leq {\lambda^\star}{B}. \label{F_B*b}\end{align}
Considering any policy $\pi$:


\begin{align*}\alpha\lambda^\star\EE{\sum_{t=1}^{\tau_B-1} {\EE{\be_{S_t\cup\sset{0}}\transpose \bC}}} &\geq
     \alpha\lambda^\star\sum_{t\geq 1} \EE{{\EE{\be_{S_t\cup\sset{0}}\transpose \bC}}\II{B_{t-1}\geq 1 +\abs{V}}} & B_{t-1}\geq 1 +\abs{V}\imp B_{t}\geq 0 \
      \\&= \alpha\lambda^\star\sum_{t\geq 1} \EE{\pa{\be_{S_t}\transpose \bC_t+\Fixed{t}}\II{B_{t-1}\geq 1 +\abs{V}}} & \text{conditioning on }\cH_{t-1}
      \\&\geq \alpha\lambda^\star B-\alpha\lambda^\star\pa{1 +\abs{V}} &\text{definition of }\tau_B.
      \\&\geq \alpha{F_{B}^\star }-2\alpha\lambda^\star\pa{1 +\abs{V}}&\text{inequality }\eqref{F_B*b},
 \end{align*}
 and
\begin{align*}\EE{\sum_{t=1}^{\tau_B-1}{\EE{\sigma\pa{S_t;\bW}}} }&\leq
    \EE{\sum_{t=1}^{\tau_B}{\EE{\sigma\pa{S_t;\bW}}} } \\&=  \sum_{t\geq 1} \EE{{\EE{\sigma\pa{S_t;\bW}}}\II{B_{t-1}\geq 0}}
    \\&= 
    \sum_{t\geq 1} \EE{\sigma\pa{S_t;\bW_t}\II{B_{t-1}\geq 0}} 
    & \text{conditioning on }\cH_{t-1}
    \\&\leq 
    \sum_{t\geq 1} \EE{\sigma\pa{S_t;\bW_t}\II{B_{t}\geq 0}} +\abs{V}
    \\&=
    F_B\pa{\pi} +\abs{V}.
\end{align*}
Again adding these two inequalities, we get the desired lower bound:
\[\EE{\sum_{t=1}^{\tau_B-1} \Delta\pa{S_t}{}} = \alpha\lambda^\star\EE{\sum_{t=1}^{\tau_B-1} {\EE{\be_{S_t\cup\sset{0}}\transpose \bC}}} - \EE{\sum_{t=1}^{\tau_B-1}{\EE{\sigma\pa{S_t;\bW}}} } \geq R_{B,\varepsilon}\pa{\pi} -2\alpha\lambda^\star\pa{1 +\abs{V}} -\abs{V}.\]
\end{proof}

\section{Proof of Theorem~\ref{thm:CUCBregret}}\label{app:CUCBregret}
\begin{proof}
Let $\alpha={1-{1}/{e}-\varepsilon}$, and $t\geq 1$.
From Proposition~\ref{prop:budgeted_regret}, we have to upper bound \[\EE{\sum_{t=1}^{\tau_B-1} \Delta\pa{S_t}}.\]
Fix $t\geq 1$. We consider the following events:
\[\mathfrak{W}_{t}\triangleq \sset{\forall ij\in E,~0\leq w_{ij}^\star-w_{ij,t}\leq 2\sqrt{\frac{1.5\log\pa{t}}{\counterw{ij}{t-1}}}},\]
\[\mathfrak{C}_{t}\triangleq \sset{\forall i\in V\cup\sset{0},~0\leq c_i^\star-c_{i,t}\leq 2\sqrt{\frac{1.5\log\pa{t}}{\counterc{i}{t-1}}}}.\]
We also consider the event $\mathfrak{A}_{t}$ under which the $\alpha-$approximation in Algorithm~\ref{algo:CUCB} holds. We already saw that \[\PP{\neg\mathfrak{A}_{t}}\leq \frac{1}{t\log^2\pa{t}}\cdot \]
From Hoeffding inequality, $\mathfrak{C}_{t}$ doesn't hold with probability bounded by $2\pa{\abs{V}+1}/{t^2},$ and $\mathfrak{W}_{t}$ doesn't hold with probability bounded by $2\abs{E}/t^2$. Thus, under the event that either $\mathfrak{C}_{t}$, $\mathfrak{W}_{t}$ or $\mathfrak{A}_{t}$ doesn't hold, then the regret is bounded by a constant (since we have a convergent series). 

We thus now assume that $\mathfrak{C}_{t}$, $\mathfrak{W}_{t}$ and $\mathfrak{A}_{t}$ hold. In particular, using $\mathfrak{C}_{t}$ and $\mathfrak{W}_{t}$, we can write
$$\lambda^\star=\frac{\sigma\pa{S^\star;\bw^\star}}{\be_{S^\star\cup\sset{0}}\transpose\bc^\star}\leq \frac{\sigma\pa{S^\star;\bw_t}}{\be_{S^\star\cup\sset{0}}\transpose\bc_t}\cdot$$
We can use this relation to write
\begin{align*}
    \Delta\pa{S_t}&=\lambda^\star\alpha\pa{\be_{S_t}\transpose \bc^\star+\fixed} - {\sigma\pa{S_t;\bw^\star}}\\&=
    \lambda^\star\alpha\pa{\pa{\be_{S_t}\transpose \bc^\star+\fixed} - \frac{\sigma\pa{S_t;\bw_t}}{\alpha\lambda^\star}} +\pa{\sigma\pa{S_t;\bw_t}-{\sigma\pa{S_t;\bw^\star}}}\\&\leq
    \lambda^\star\alpha\pa{\pa{\be_{S_t}\transpose \bc^\star+\fixed} - \frac{
    \sigma\pa{S_t;\bw_t}}{\alpha{\sigma\pa{S^\star;\bw_t}}}{\pa{{\be_{S^\star}\transpose \bc_t+c_{0,t}}}}} +\pa{\sigma\pa{S_t;\bw_t}-{\sigma\pa{S_t;\bw^\star}}}\\
    &\leq
    \lambda^\star\alpha\pa{\pa{\be_{S_t}\transpose \bc^\star+\fixed} - {\pa{{\be_{S_t}\transpose \bc_t+c_{0,t}}}}} +\pa{\sigma\pa{S_t;\bw_t}-{\sigma\pa{S_t;\bw^\star}}},
\end{align*}
where the last inequality is from $\mathfrak{A}_{t}$. Notice here that in the case the costs are known, the first term in this bound disappears, and we can then safely take $\lambda^\star=$0, explaining why in the final bound the term in front of $\lambda^\star$ disappears. We now use Fact~\ref{wang}, and then $\mathfrak{C}_{t}$, $\mathfrak{W}_{t}$ to further get the bound
$$\Delta\pa{S_t}\leq \underbrace{\lambda^\star\alpha\sum_{i\in S_t}{\pa{1\wedge 2\sqrt{\frac{1.5\log\pa{t}}{\counterc{i}{t-1}}}}}}_{\numterm{term1}} +  \underbrace{\abs{V}\sum_{ij\in E} p_i(S_t;\bw^\star)\pa{1\wedge 2\sqrt{\frac{1.5\log\pa{t}}{\counterw{ij}{t-1}}}}}_{\numterm{term2}}\cdot$$

Then, necessarily either $\Delta(S_t)\leq 2\cdot\eqref{term1}$ or $\Delta(S_t)\leq 2\cdot\eqref{term2}$ is true. The first event can be handle exactly as in standard combinatorial semi bandit settings, using the following upper bound on the expectation of the random horizon \citep{pmlr-v89-perrault19a}:
\[\EE{{\tau_B -1}}\leq \pa{2B/\fixed+1}^2.\]
This allows us to get a term of order \[\lambda^\star\log(B/\fixed)\sum_{i\in V
} \frac{\abs{V}}{\Delta_{i,\min}}\CommaBin\] in the regret upper bound.

For the second event, the analysis of \cite{wang2017improving} uses triggering probability group to deal with $p_i(S_t;\bw^\star)$. We propose here another method, simpler, that allows us to transform the factor $\abs{E}$, present in the bound of \cite{wang2017improving}, into $\max_{S\subset V}\sum_{i\in V}d_i p_i(S;\bw^\star)$, a potentially much smaller quantity. More precisely, we first use a reverse amortisation: 
\begin{align*}\Delta(S_t)\leq -\Delta(S_t) + 4\cdot\eqref{term2} &= 4\abs{V}\sum_{ij\in E} p_i(S_t;\bw^\star)\pa{1\wedge2\sqrt{\frac{1.5\log\pa{t}}{\counterw{ij}{t-1}}}-\frac{\Delta(S_t)}{4\abs{V}\sum_{k\in V}d_k p_k(S_t;\bw^\star)}}\\&\leq 4\abs{V}\sum_{ij\in E} p_i(S_t;\bw^\star)\II{\counterw{ij}{t-1} \leq 1.5\log(t)\pa{\frac{8\abs{V}\sum_{k\in V}d_k p_k(S_t;\bw^\star)}{\Delta(S_t)}}^2}\pa{1\wedge2\sqrt{\frac{1.5\log\pa{t}}{\counterw{ij}{t-1}}}}.\end{align*}
Since we have
\[\EE{\sum_{t=1}^{\tau_B-1} \Delta\pa{S_t}} \leq \EE{\sum_{t=1}^{\tau_B} \EEcc{\Delta\pa{S_t}}{\cH_{t-1}}},\]
and that 
$$\EEcc{p_i(S_t;\bw^\star)\II{\counterw{ij}{t-1} \leq 1.5\log(t)\pa{\frac{8\abs{V}\sum_{k\in V}d_k p_k(S_t;\bw^\star)}{\Delta(S_t)}}^2}{\pa{1\wedge 2\sqrt{\frac{1.5\log\pa{t}}{\counterw{ij}{t-1}}}}}}{\cH_{t-1}}$$
equals
$$\EEcc{\II{S_t\reach{\bw^\star} i}\II{\counterw{ij}{t-1} \leq 1.5\log(t)\pa{\frac{8\abs{V}\sum_{k\in V}d_k p_k(S_t;\bw^\star)}{\Delta(S_t)}}^2}\pa{1\wedge2\sqrt{\frac{1.5\log\pa{t}}{\counterw{ij}{t-1}}}}}{\cH_{t-1}},$$
it is sufficient to bound the quantity
$$\sum_{t=1}^{\tau_B} 4\abs{V}\sum_{ij\in E,~p_i(S_t;\bw^\star)>0} \II{S_t\reach{\bw^\star} i}\II{\counterw{ij}{t-1} \leq 1.5\log(t)\pa{\frac{8\abs{V}\sum_{k\in V}d_k p_k(S_t;\bw^\star)}{\Delta(S_t)}}^2}{\pa{1\wedge2\sqrt{\frac{1.5\log\pa{t}}{\counterw{ij}{t-1}}}}}\cdot$$

Therefore, counters $\counterw{ij}{t-1}$ are ensured to increase thanks to the event $\sset{S_t\reach{\bw^\star} i}$. We can now handle this exactly as in standard combinatorial semi bandit setting, to get a bound of order

$$\log(B/c_0^\star)\sum_{i\in V}d_i \frac{\abs{V}^2 \max_{S\subset V,~p_i(S;\bw^\star)>0}\sum_{k\in V}d_k p_k(S;\bw^\star) }{\Delta_{i,\min}}.$$
\end{proof}

\paragraph{Problem-independent bound}
The problem-independent bound of
$\cO\pa{\abs{V}\sqrt{B\log B\sum_{i\in V} d_i p_{i,\max}}}$
 is an immediate consequence of our problem-dependent bound, decomposing, classically, the regret in two terms by filtering by whether or not $\Delta(S_t)\leq \delta$, and then taking the worst regime for $\delta$.

\section{Proof of Proposition~\ref{prop:counter}}\label{app:counter}
\label{app:regret}
\begin{proof} First, notice that we trivially have
\[ \PP{\mathfrak{P}_{t}\text{ and }{p_i\pa{\sset{j};\bw^\star}}\leq \frac{8\delta\pa{t}}{\counterc{j}{t-1}}\text{ and }\frac{\delta\pa{t}p_i\pa{\sset{j};\bw^\star}}{\counterw{i}{t-1}}> \frac{8\delta\pa{t}}{\counterc{j}{t-1}}}=0.\]
Thus, 
let's  prove that
\[\PP{\mathfrak{P}_{t}\text{ and }{p_i\pa{\sset{j};\bw^\star}}> \frac{8\delta\pa{t}}{\counterc{j}{t-1}}\text{ and }\frac{\delta\pa{t}p_i\pa{\sset{j};\bw^\star}}{\counterw{i}{t-1}}> \frac{8\delta\pa{t}}{\counterc{j}{t-1}}}\leq 1/t^2.\]
We define another counter for $\pa{i,j}\in V^2$ as follows:\[N_{i,j,t-1}\triangleq \sum_{t'=1}^{t-1}\II{j\in S_{t'},~\sset{j}\reach{\bW_{t'}}i}.\]
Note that we have $N_{i,j,t-1} \leq \pa{ \counterw{i}{t-1}\wedge\counterc{j}{t-1}}$.
We can thus remove $\mathfrak{P}_{t}$ and replace $\counterw{i}{t-1}$ by $N_{i,j,t-1}$, since this can only increases the probability. By an union bound we have, 
\begin{align*}
    \PP{{p_i\pa{\sset{j};\bw^\star}}> \frac{8\delta\pa{t}}{\counterc{j}{t-1}}\text{ and }\frac{p_i\pa{\sset{j};\bw^\star}}{N_{i,j,t-1}}> \frac{8}{\counterc{j}{t-1}}}&\leq \sum_{t'>\frac{8\delta\pa{t}}{p_i\pa{\sset{j};\bw^\star}}}\PP{\counterc{j}{t-1}=t',~\frac{t'p_i\pa{\sset{j};\bw^\star}}{8}> N_{i,j,t-1}}.
\end{align*}
Since the random variables $\II{\sset{j}\reach{\bW_{t''}}i}$ are bernouillies of mean $p_i\pa{\sset{j};\bw^\star},$ we can apply the Fact~\ref{chernoff} to get
\[\PP{\counterc{j}{t-1}=t',~\frac{t'p_i\pa{\sset{j};\bw^\star}}{8}> N_{i,j,t-1}}\leq \exp\pa{-\pa{7/8}^28\delta\pa{t}/2}<1/t^3.\]
By taking $t'$ over $\sset{0,\dots,t-1}$, the proposition holds. 
\begin{fact}[Multiplicative Chernoff Bound \citep{mitzenmacher2017probability}]\label{chernoff}
Let $X_1,\dots,X_t$ be Bernoulli random variables, of parameter $\mu$, then for $Y=X_1+\cdots+X_t$, we have with $\delta\in (0,1),$\[\PP{Y\leq \pa{1-\delta}t\mu }\leq e^{-\delta^2 t \mu/2}.\]
\end{fact}
\end{proof}

\section{Proof of Proposition~\ref{prop:greedy}}\label{app:greedy}
\begin{proof} In the proof, we use the notation $\sigma\ret{i}{S}\triangleq \sigma\pa{\sset{i}\cup S}-\sigma\pa{S}$. For any $k\in [\abs{V}]$, 
\begin{align*}\nonumber
    \sigma\pa{S^\star} - \sigma\pa{S_{k-1}} &\leq \sum_{i\in S
    ^\star\backslash S_{k-1} }\sigma\ret{i}{S_{k-1}} &\text{Submodularity, monotonicity of }\sigma
    \\&\leq\nonumber \frac{\sigma\ret{i_k}{S_{k-1}} }{c_{i_k}}
    \sum_{i\in S
    ^\star\backslash S_{k-1} } c_{i} &\text{Algorithm~\ref{algo:greedy}}
    \\&\leq \nonumber \frac{\sigma\ret{i_k}{S_{k-1}} }{c_{i_k}}
    \sum_{i\in S
    ^\star } c_{i}.
\end{align*}
i.e., for all $k\in[\abs{V}]$ such that 
$\sigma\pa{S^\star} - \sigma\pa{S_{k-1}}\geq 0$,
\begin{align}\frac{c_{i_{k}}}{\be_{S^\star}\transpose\bc}\leq \frac{  \sigma\ret{i_k}{S_{k-1}} }{\sigma\pa{S^\star} - \sigma\pa{S_{k-1}}}
   \cdot \label{last}
\end{align}

There must be an index $\ell\in\sset{0,1,\dots,\abs{V}-1}$ such that $\be_{S_\ell}\transpose\bc\leq \be_{S^\star}\transpose\bc\leq \be_{S_{\ell+1}}\transpose\bc$. Let $\beta\in [0,1]$ be such that \begin{align}\be_{S^\star}\transpose\bc = \pa{1-\beta}\be_{S_\ell}\transpose\bc + \beta \be_{S_{\ell+1}}\transpose\bc.\label{combconv}\end{align}  If
$\sigma\pa{S^\star} - \pa{1-\beta}\sigma\pa{S_{\ell}} - \beta \sigma\pa{S_{\ell+1}}\leq 0$, then we have 
\[\pa{1-e^{-1}}\frac{\sigma\pa{S^\star}}{\be_{S^\star}\transpose\bc }\leq\frac{\sigma\pa{S^\star}}{\be_{S^\star}\transpose\bc }\leq \frac{\pa{1-\beta}\sigma\pa{S_{\ell}} + \beta \sigma\pa{S_{\ell+1}}}{\pa{1-\beta}\be_{S_\ell}\transpose\bc + \beta \be_{S_{\ell+1}}\transpose\bc}\cdot\] Else, \begin{align}\sigma\pa{S^\star} - \pa{1-\beta}\sigma\pa{S_{\ell}} - \beta \sigma\pa{S_{\ell+1}}> 0\quad\text{and}\quad\sigma\pa{S^\star} - \sigma\pa{S_{k}}> 0\quad\text{for all }k \in[\ell],\label{pos}\end{align} so
we can write the following, 
\begin{align*}\frac{  \sigma\pa{S^\star} - (1-\beta) \sigma\pa{S_{\ell}} -\beta \sigma\pa{S_{\ell+1}}}{\sigma\pa{S^\star}}&\leq
  \frac{  \sigma\pa{S^\star} - (1-\beta) \sigma\pa{S_{\ell}} -\beta \sigma\pa{S_{\ell+1}} }{\sigma\ret{S^\star}{\emptyset}} \\&=\frac{  \sigma\pa{S^\star} - \sigma\pa{S_{\ell}} -\beta \sigma\ret{i_{\ell+1}}{S_\ell} }{  \sigma\pa{S^\star} - \sigma\pa{S_{\ell}}  }\prod_{k\in[\ell]}\frac{  \sigma\pa{S^\star} - \sigma\pa{S_{k}} }{\sigma\pa{S^\star}- \sigma\pa{S_{k-1}}} 
  \\&= \pa{1-\frac{  \beta \sigma\ret{i_{\ell+1}}{S_{\ell}} }{\sigma\pa{S^\star} - \sigma\pa{S_{\ell}}}}
  \prod_{k\in[\ell]}\pa{1-\frac{  \sigma\ret{i_k}{S_{k-1}} }{\sigma\pa{S^\star} - \sigma\pa{S_{k-1}}}}
  \\&\leq
  \pa{1-\frac{\beta c_{i_{\ell+1}}}{\be_{S^\star}\transpose\bc}}
  \prod_{k\in[\ell]}\pa{1-\frac{ c_{i_{k}}{}}{\be_{S^\star}\transpose\bc}}&\eqref{last}\text{ and }\eqref{pos}
  \\&\leq
  \exp\pa{-\frac{{ \beta c_{i_{\ell+1}}{} + \sum_{k\in[\ell]} c\pa{i_{k}}{}}}{\be_{S^\star}\transpose\bc}} &1-x\leq e^{-x}
    \\&= \exp\pa{-\frac{{\pa{1-\beta}\be_{S_\ell}\transpose\bc + \beta \be_{S_{\ell+1}}\transpose\bc}}{\be_{S^\star}\transpose\bc}}
   =e^{-1}. &\eqref{combconv}
\end{align*}
Rearranging the inequality, we obtain the following:
\begin{align}
    \pa{1-e^{-1}}\sigma\pa{S^\star} \leq  (1-\beta) \sigma\pa{S_{\ell}} +\beta \sigma\pa{S_{\ell+1}}.
    \label{final}
\end{align}
i.e.,\[\pa{1-e^{-1}}\frac{\sigma\pa{S^\star}}{\be_{S^\star\cup\sset{0}}\transpose \bc }\leq \frac{\pa{1-\beta}\sigma\pa{S_{\ell}} + \beta \sigma\pa{S_{\ell+1}}}{\pa{1-\beta}\be_{S_\ell\cup\sset{0}}\transpose\bc + \beta \be_{S_{\ell+1}\cup\sset{0}}\transpose\bc}\cdot\]
The output $S$ of Algorithm~\ref{algo:greedy} maximizes the ratio of $\sigma\pa{S_k}/\be_{S_k\cup\sset{0}}\transpose\bc$ over $k$. 
Thus,
\[\max_{k\leq \ell+1}\frac{\sigma\pa{S_k}}{\be_{S_k\cup\sset{0}}\transpose\bc }\leq \frac{\sigma\pa{S}}{\be_{S\cup\sset{0}}\transpose \bc }\cdot\]
We end the proof remarking that
\[\max_{k\in \sset{\ell,\ell+1}}\frac{\sigma\pa{S_k}}{\be_{S_k\cup\sset{0}}\transpose\bc }\geq \frac{\pa{1-\beta}\sigma\pa{S_{\ell}} + \beta \sigma\pa{S_{\ell+1}}}{\pa{1-\beta}\be_{S_\ell\cup\sset{0}}\transpose\bc + \beta \be_{S_{\ell+1}\cup\sset{0}}\transpose\bc}\cdot\]
\end{proof}

\section{Proof of Theorem~\ref{thm:regret}}\label{app:regret}
\begin{proof}
Let $\alpha={1-{1}/{e}-\varepsilon}$, and $t\geq 1$.
From Proposition~\ref{prop:budgeted_regret}, we have to upper bound \[\EE{\sum_{t=1}^{\tau_B-1} \Delta\pa{S_t}}.\]
In the proof, in addition to $\mathfrak{P}_{t}\triangleq\sset{\forall i\in V, {\counterw{i}{t-1}} \geq \delta\pa{t}}$, we consider the following events:
\[\mathfrak{W}_{t}\triangleq \sset{\sum_{ij\in E}\counterw{i}{t-1}\pa{{w_{ij}^\star-\meanw{ij,t-1}}}^2 \leq 2 \delta\pa{t}},\]
\[\mathfrak{C}_{t}\triangleq \sset{\forall i\in V\cup\sset{0},~0\leq c_i^\star-c_{i,t}\leq 1\wedge 2\sqrt{\frac{1.5\log\pa{t}}{\counterc{i}{t-1}}}}.\]
\[\mathfrak{B}_{t}\triangleq \sset{\forall i,j\in V,~\frac{\delta\pa{t}p_i\pa{\sset{j};\bw^\star}}{\counterw{i}{t-1}}\leq \frac{8\delta\pa{t}}{\counterc{j}{t-1}}}.\]
As previously, we have the following upper bound on the expectation of the random horizon:
\[\EE{{\tau_B -1}}\leq \pa{2B/\fixed+1}^2.\]
For each node $i\in V,$ if $\counterw{i}{t-1}\geq \delta\pa{\tau_B-1}$, then $i$ will not be intentionally added to the seed set in \textsc{boim-cucb}$_5$.  Then, each node is intentionally added
for at most $\delta\pa{\tau_B-1}+1$ times. Thus, we can write
\[\EE{\sum_{t=1}^{\tau_B-1} \Delta\pa{S_t}\II{\neg\mathfrak{P}_{t}}}\leq \EE{\pa{\delta\pa{\tau_B-1}+1}\abs{V}\lambda^\star\pa{\abs{V}+1}}\leq \pa{\delta\pa{\pa{2B/\fixed+1}^2}+1}\abs{V}\lambda^\star\pa{\abs{V}+1}.\]
We can therefore assume that $\mathfrak{P}_{t}$ holds. In this case, we have by Proposition~\ref{prop:counter} that $\mathfrak{B}_{t}$ doesn't hold with probability bounded by $\abs{V}^2/t^2$. On the other hand, from Fact~\ref{degenne},
 $\mathfrak{W}_{t}$ doesn't hold with probability bounded by ${1}/\pa{t\log^2\pa{t}}$, and from Hoeffding inequality, $\mathfrak{C}_{t}$ doesn't hold with probability bounded by $2\pa{\abs{V}+1}/{t^2}.$ We can consider the event $\mathfrak{A}_{t}$ under which the $\alpha-$approximation in \textsc{boim-cucb}$_5$ holds. We already saw that \[\PP{\neg\mathfrak{A}_{t}}\leq \frac{1}{t\log^2\pa{t}}\cdot \]
The regret in the case one of the events $\neg\mathfrak{A}_{t},\neg\mathfrak{B}_{t},\neg\mathfrak{W}_{t},\neg\mathfrak{C}_{t}$ holds is thus bounded by a constant depending on~$\abs{V}$ and $\lambda^\star$.
It thus remains to upper bound \[\EE{\sum_{t=1}^{\tau_B-1} \Delta\pa{S_t}\II{\mathfrak{A}_{t},\mathfrak{B}_{t},\mathfrak{W}_{t},\mathfrak{C}_{t}}}.\]
For this, notice that from $\mathfrak{A}_{t}$, $S_t$ which is the seed set chosen by our policy at round $t$, is an $\alpha$-approximate maximizer of $A\mapsto f(A)/\pa{{\be_{A}\transpose \bc_t+c_{0,t}}}$, where $f$ is one of the optimistic spreads considered in \textsc{boim-cucb}$_5$. We thus have
\[ \frac{f(S_t)}{{{\be_{S_t}\transpose \bc_t+c_{0,t}}}} \geq \alpha\frac{f(S^\star)}{{{\be_{S^\star}\transpose \bc_t+c_{0,t}}}}\CommaBin\]
where $S^\star\in \argmax_{S\subset V}\frac{{\sigma\pa{S;\bw^\star}}}{{{\be_{S}\transpose \bc^\star+c^\star_{0}}}}\cdot$
Since under $\mathfrak{W}_{t},$ $f\pa{S^\star}\geq \sigma\pa{S^\star;\bw^\star}$, we can derive the following upper bound on the gap:
\begin{align*}
    \Delta\pa{S_t}&=\lambda^\star\alpha\pa{\be_{S_t}\transpose \bc^\star+\fixed} - {\sigma\pa{S_t;\bw^\star}}\\&=
    \lambda^\star\alpha\pa{\pa{\be_{S_t}\transpose \bc^\star+\fixed} - \frac{f\pa{S_t}}{\alpha\lambda^\star}} +\pa{f\pa{S_t}-{\sigma\pa{S_t;\bw^\star}}}\\&\leq
    \lambda^\star\alpha\pa{\pa{\be_{S_t}\transpose \bc^\star+\fixed} - \frac{f\pa{S_t}}{{\alpha f(S^\star)}}{\pa{{\be_{S^\star}\transpose \bc_t+c_{0,t}}}}} +\pa{f\pa{S_t}-{\sigma\pa{S_t;\bw^\star}}}&\mathfrak{W}_{t},\mathfrak{C}_{t}\\
    &\leq
    \lambda^\star\alpha\pa{\pa{\be_{S_t}\transpose \bc^\star+\fixed} - {\pa{{\be_{S_t}\transpose \bc_t+c_{0,t}}}}} +\pa{f\pa{S_t}-{\sigma\pa{S_t;\bw^\star}}}.
    &\mathfrak{A}_{t}
\end{align*}
From this point, we can use the condition satisfied by $f$ in \textsc{boim-cucb}$_5$:
\[f\pa{S_t}\leq {\sigma\pa{S_t;\vmeanw{t-1}} + \bonus_5\pa{S_t}}.\]
Using Fact~\ref{wang} with Fact~\ref{degenne}, we can further have with Cauchy-Schwartz inequality
\[{\sigma\pa{S_t;\vmeanw{t-1}}}- {\sigma\pa{S_t;\bw^\star}}\leq \bonus_5\pa{S_t}.\]
This allows us to get, using $\mathfrak{C}_{t},$
\begin{align*}
    \Delta\pa{S_t}&\leq \lambda^\star\alpha \sum_{i\in S_t\cup\sset{0}} 1\wedge 2\sqrt{\frac{1.5\log\pa{t}}{\counterc{i}{t-1}}} + 2\bonus_5\pa{S_t}.
    \end{align*}
    Since we have $\counterc{0}{t-1}=t,$ we can remove $\sset{0}$ in $S_t\cup\sset{0}$, by looking at the regret under the event $\sset{2\lambda^\star\alpha\sqrt{1.5\log\pa{t}/t}\leq \Delta\pa{S_t}/2}$, giving
    \begin{align}\label{gap_bound}
    \Delta\pa{S_t}&\leq 2\lambda^\star\alpha \sum_{i\in S_t} 1\wedge 2\sqrt{\frac{1.5\log\pa{t}}{\counterc{i}{t-1}}} + 4\bonus_5\pa{S_t}.
    \end{align}
    The regret upper bound in the case this event doesn't hold is bounded by a constant depending on the inverse of the squared minimum gap and $\lambda^\star$. The first part in \eqref{gap_bound} can be handle exactly as in standard combinatorial budgeted semi bandit settings, 
to get a term of order \[\lambda^\star\log(B/\fixed)\sum_{i\in V
} \frac{\abs{V}}{\Delta_{i,\min}}\CommaBin\] in the regret upper bound. We can use the analysis of \citet{Degenne2016} to deal with the second part, to get a term of order
\[\delta\pa{B/\fixed}\abs{V}^2\abs{E}\sum_{i\in V}\frac{\log^2\pa{\abs{V}}}{\Delta_{i,\min}}\CommaBin\]
in the regret upper bound. We thus get the desired result.
\end{proof}

\section{Proof of Theorem~\ref{thm:low}}\label{app:low}
\begin{proof}Let $\alpha={1-{1}/{e}-\varepsilon}$, and $t\geq 1$.
The beginning of the proof is the same as in Theorem~\ref{thm:regret}, except we no longer consider the event $\mathfrak{B}_{t}$, and we consider a new event:
\[\mathfrak{R}_{t}\triangleq \sset{\forall i\in V, \counterw{i}{t-1}\geq  \abs{E}\delta\pa{t}}.\]
As for Theorem~\ref{thm:regret}:
\begin{itemize}
    \item 
The regret in the case $\mathfrak{R}_{t}$ doesn't hold can be bounded by a term of order \[\lambda^\star{\abs{V}^2}\abs{E}\log\pa{B}.\]
\item
When all the events hold, the same analysis gives
\[\Delta\pa{S_t}\leq 2\lambda^\star\alpha \sum_{i\in S_t} 1\wedge 2\sqrt{\frac{1.5\log\pa{t}}{\counterc{i}{t-1}}} + 4\bonus_1\pa{S_t;\vmeanw{t-1}},\]
and the first term can be handled in the same way. 
\end{itemize}
The second term can be analyzed in the following way:
After bounding it by $4m\bonus\pa{S_t;\vmeanw{t-1}}$, see that using Fact~\ref{wang} on the quantity $p_i\pa{S;\vmeanw{t-1}}$ present in this bonus, we get
$$p_i\pa{S_t;\vmeanw{t-1}}\leq p_i\pa{S_t;\bw^\star}+\frac{1}{\abs{V}}\bonus\pa{S_t;\bw^\star}.$$
By subadditivity, and from $\mathfrak{R}_{t}$, we have
\begin{align*}4m\bonus\pa{S_t;\vmeanw{t-1}}&\leq4m\bonus\pa{S_t;\bw^\star}+ 4m\abs{V}\sqrt{\delta\pa{t}\sum_{i\in V}d_i\frac{\bonus\pa{S_t;\bw^\star}^2}{\abs{V}^2\counterw{i}{t-1}}}\\&\leq
4m\bonus\pa{S_t;\bw^\star}+ 4m\abs{V}\sqrt{\sum_{i\in V}d_i\frac{\bonus\pa{S_t;\bw^\star}^2}{\abs{V}^2\abs{E}}}
=8m\bonus\pa{S_t;\bw^\star}.
\end{align*}
We can now use the analysis of \citet{Degenne2016}, together with the one from \citet{wang2017improving} to deal with probabilistically triggered arms, to get in the regret upper bound a term of order
\[\delta\pa{B/\fixed}m^2\abs{V}^2\sum_{i\in V}d_i\frac{\log^2\pa{\abs{E}}}{\Delta_{i,\min}}\cdot\]
\vspace{-.6cm}

\noindent
\end{proof}

\section{\textsc{skim} for influence maximization with cost}\label{app:sketch}

In this section, we provide an adaptation of \textsc{skim} \citep{cohen2014sketch} to our ratio maximization setting. 
Let $\pa{\bw,\textsc{bonus}}=\pa{\bw_t,0}$ or $\pa{\vmeanw{t-1},\bonus_5}$ (depending if we want to maximize the \textsc{cucb} ratio or the $\bonus_5$ based ratio), \textsc{skim} for IC with weights $\bw$ can be used \citep{cohen2014sketch}, but instead of taking $k-1$ as the threshold for the length of the sketch of node $i$ (i.e. $i$ is chosen as soon as $\abs{\text{sketch}[i]}>k-1$), we rather consider  
\[{k c_i - {\frac{{\textsc{bonus}\pa{\sset{i}\cup S}-\textsc{bonus}\pa{S}}}{\abs{V}}\text{sketch}[i][-1]}},\]
where $\text{sketch}[i][-1]$ is the last added rank in $\text{sketch}[i]$ ($0$ if the sketch is empty), and $S$ is the current seed set built so far.
This way, we can estimate the (non-optimistic) marginal gain spread of the chosen node $i$ as
\[\frac{k \abs{V} c_i - {\pa{{\textsc{bonus}\pa{\sset{i}\cup S}-\textsc{bonus}\pa{S}}}{}\text{sketch}[i][-1]}}{\text{sketch}[i][-1]}\cdot\]
We get the optimistic version by adding ${{\textsc{bonus}\pa{\sset{i}\cup S}-\textsc{bonus}\pa{S}}}$, i.e., it is
\(\frac{k \abs{V} c_i }{\text{sketch}[i][-1]}\cdot\)
Finally, we get the ratio marginal gain by dived by $c_i$, giving $\frac{k \abs{V}  }{\text{sketch}[i][-1]}$. 
We do maximize the ratio marginal gain by doing this all procedure, because the ranks are examined in ascending order.
Notice we have to normalize costs at the beginning of the loop for finding $i$, such that each of the threshold are greater than $k-1$, to ensure that the length of the chosen sketch is at least $k$ and that our estimation hold with high probability.

It should be noted that the main difficulty is to go through the ranks in the right order while taking into account the cost. An alternative would have been to draw classical ranks and change them at each update of the sketch of $u$ according to the cost of $u$. This procedure works for estimating the spread, but we can't guarantee that the ranks are in the descending order of the marginal ratio, so all the marginal ratios must be estimated before selecting the maximizer.

Adapting Lemma 4.2 from \citet{cohen2014sketch}, we obtain that the expected total number of rank insertions at a particular node is $\cO\pa{k\log\pa{\abs{V}k\frac{c_{\max}}{c_{\min}}}}$, thus giving a global complexity of 
$\cO\pa{\abs{E} k\log\pa{\abs{V}k\frac{c_{\max}}{c_{\min}}}}$. We note the dependence in $c_min$, which although not desired, is only logarithmic. When all the costs are equal, we recover the standard \textsc{skim} complexity.

\section{Evaluating bonuses with sketches}
\label{app:evalbo}
In this section, we give details on the bonuses evaluation.
Notice that the optimistic spread
\begin{align}\sigma\pa{S;\vmeanw{t-1}}+\bonus_2\pa{S;\vmeanw{t-1}}=&\sum_{i\in V} p_i\pa{S;\vmeanw{t-1}}\pa{1+\abs{V}\sqrt{\frac{\delta\pa{t}d_i}{\counterw{i}{t-1}}}}\label{weighted_spread}\end{align}
\vspace{-.5cm}

\noindent
is actually a weighted spread, with weights $1+\abs{V}\sqrt{\frac{\delta\pa{t}d_i}{\counterw{i}{t-1}}}\cdot$
Thus, ranks used in the sketching have to be drawn from a distribution that
depends on these weights \citep{cohen2016minhash}. This can be done using the exponential or uniform distribution \citep{cohen1997size,cohen2007summarizing}. $\bonus_3\pa{A}$ is a square root of a weighted spread, and the same as above holds.
For $\bonus_1\pa{A}$, in addition to the above weighted consideration, with weights ${\delta\pa{t}d_i}/{\counterw{i}{t-1}},$ we have to take care of the squared probability
$p_i\pa{\sset{j};\vmeanw{t-1}}^2$. To do so, the graph is replicated in each instance of the sketching, i.e., instead of considering combined reachability sets with a single graph per instance, we consider two independent graphs per instance, and look at node-instance pairs satisfying the reachability on \emph{both} graphs.
Notice, we leverage here on the fact that we only evaluate on sets that are singletons $A=\sset{j}$.
Indeed, in this case, for a node $i\in V$, the probability that $A$ reaches $i$ on one instance, squared, is equal to the probability that some node in $A$ reaches $i$ on both instances. 

\vspace{-.2cm}
\section{Further experiments}
\label{app:exp}
In this section, we present other experiments that we conducted on a complete 10 node graph, with known costs $c_0^\star=1$, and for all $i\in V$, $c_i^\star$ is randomly drawn in $(0,1)$. We also chose $\bw^\star\sim U\pa{0,0.1}^{\otimes E}$, as in Section~\ref{sec:exp}. We compare 
 the \textsc{boim-cucb} algorithm to \textsc{boim-{cucb}-regularized}, another algorithm that might challenge \textsc{boim-cucb} in our setting. 
\textsc{boim-{cucb}-regularized} is exactly as \textsc{boim-cucb} except that the objective that is optimized is
$S\mapsto\sigma(S;\bw_t)-\lambda \be_S\bc_t,$
for $\lambda$ being an input parameter to the algorithm.
We can see that as for \textsc{boim-cucb}, this algorithm have the willingness to maximize the function $\sigma$ while minimizing the cost function. The fundamental difference is on the importance given to one or the other function, controled by $\lambda$.
 We use a greedy maximization in \textsc{boim-{cucb}-regularized}.
A greedy optimization of the objective $S\mapsto\sigma(S;\bw_t)-\lambda \be_S\bc_t$ is a heuristic which, although not supported in theory, performs well in practice. 

We run experiments over up to $T=10000$ rounds, on five different draws for $\bw^\star$ and $\bc^\star$, and $3$ different values $\lambda=2,3,4$. Results are shown in Figure~\ref{exp:regretbis}. We observe that \textsc{boim-cucb} is in general better than \textsc{boim-{cucb}-regularized}. If the variable $\lambda$ is properly chosen, performances similar to \textsc{boim-cucb} can be obtained.
This is not surprising since 
\textsc{boim-cucb} aims (but only approximatively) to select $S^\star_t\in \argmax_{S\subset V} \sigma(S;\bw_t)/ \be_{S\cup \sset{0}}$. If 
$\lambda = \sigma(S_t^\star;\bw_t)/ \be_{S_t^\star\cup \sset{0}}\bc_t$, then we also have that \textsc{boim-{cucb}-regularized} aims at choosing $S^\star_t$, since one can notice that
$S^\star_t \in \argmax_{S\subset V} \sigma(S;\bw_t)-\lambda \be_S\bc_t.$
\vspace{-.2cm}
\begin{figure}[t]
\begin{center}
\includegraphics[scale=.14]{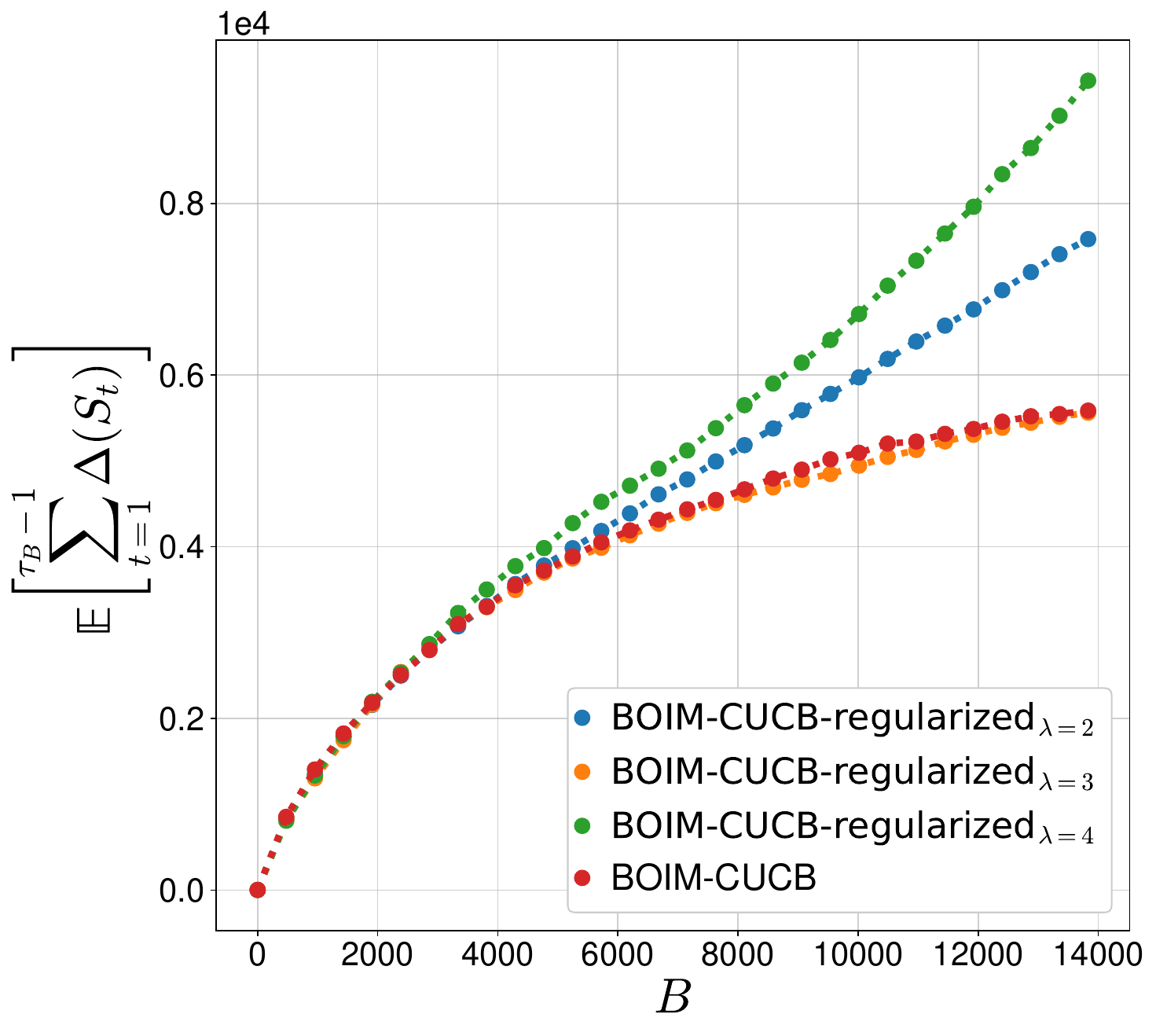}
\includegraphics[trim={3.28cm 0 0 0},clip,scale=.14]{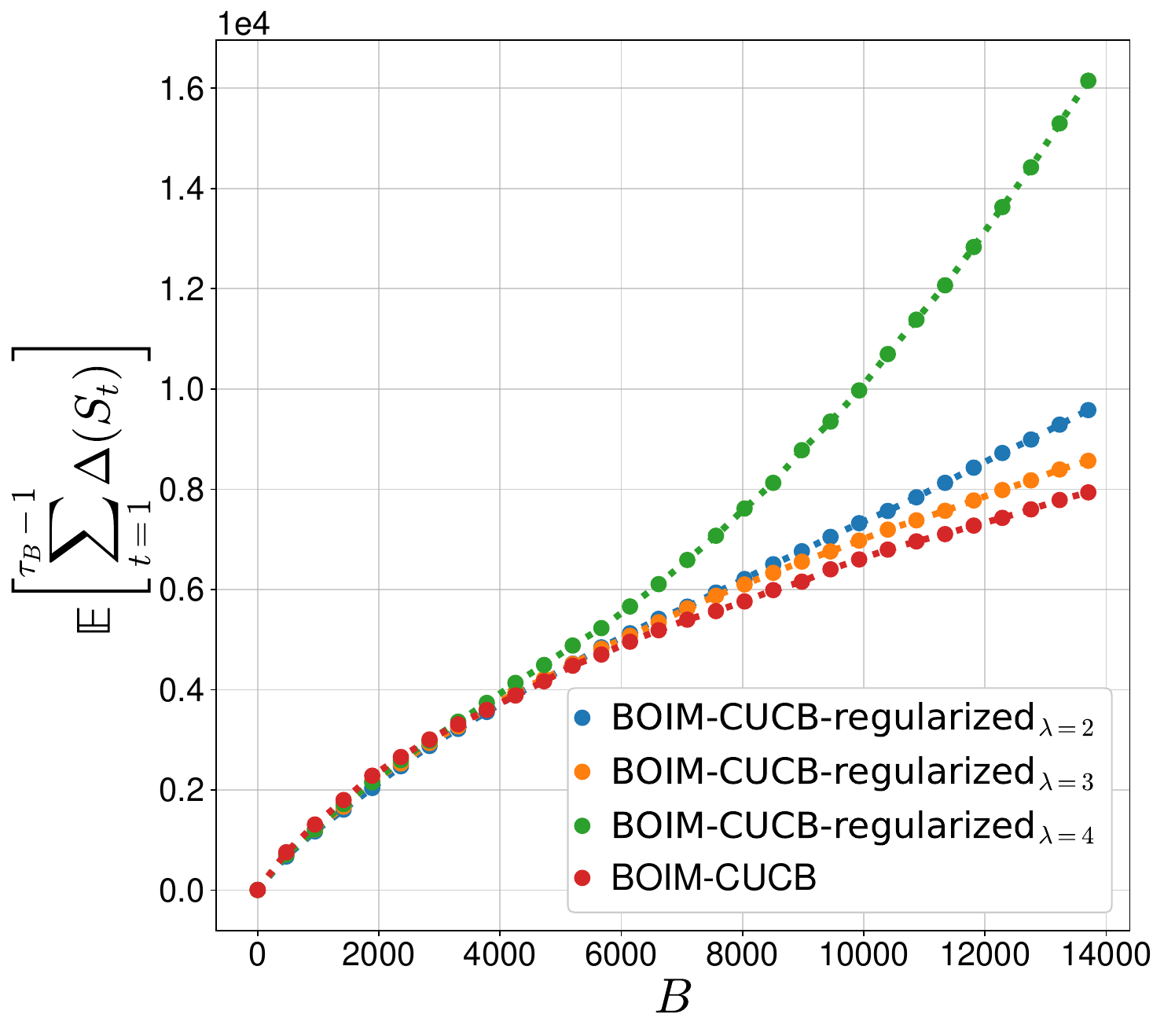}
\includegraphics[trim={3.28cm 0 0 0},clip,scale=.14]{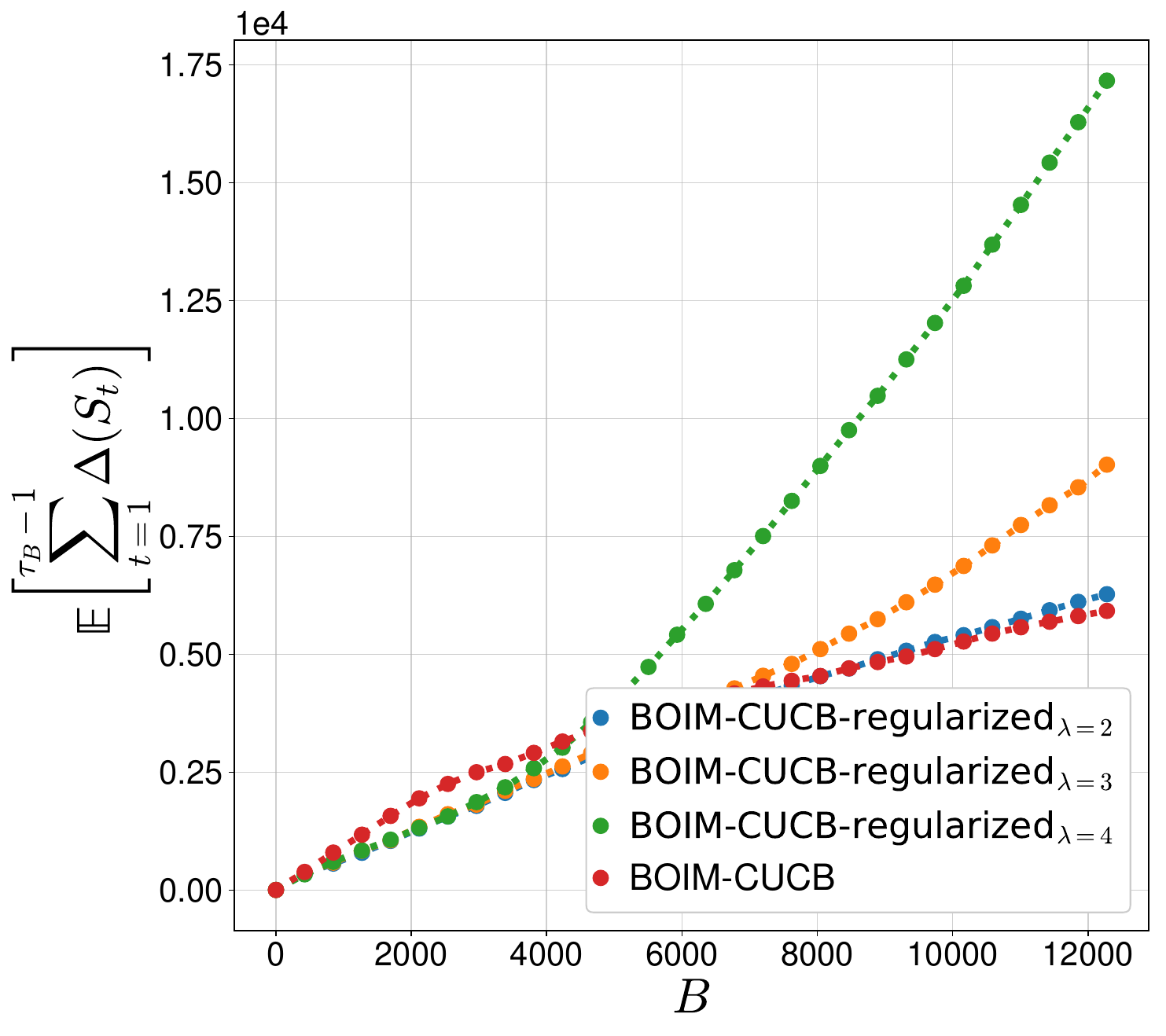}
\includegraphics[trim={3.28cm 0 0 0},clip,scale=.14]{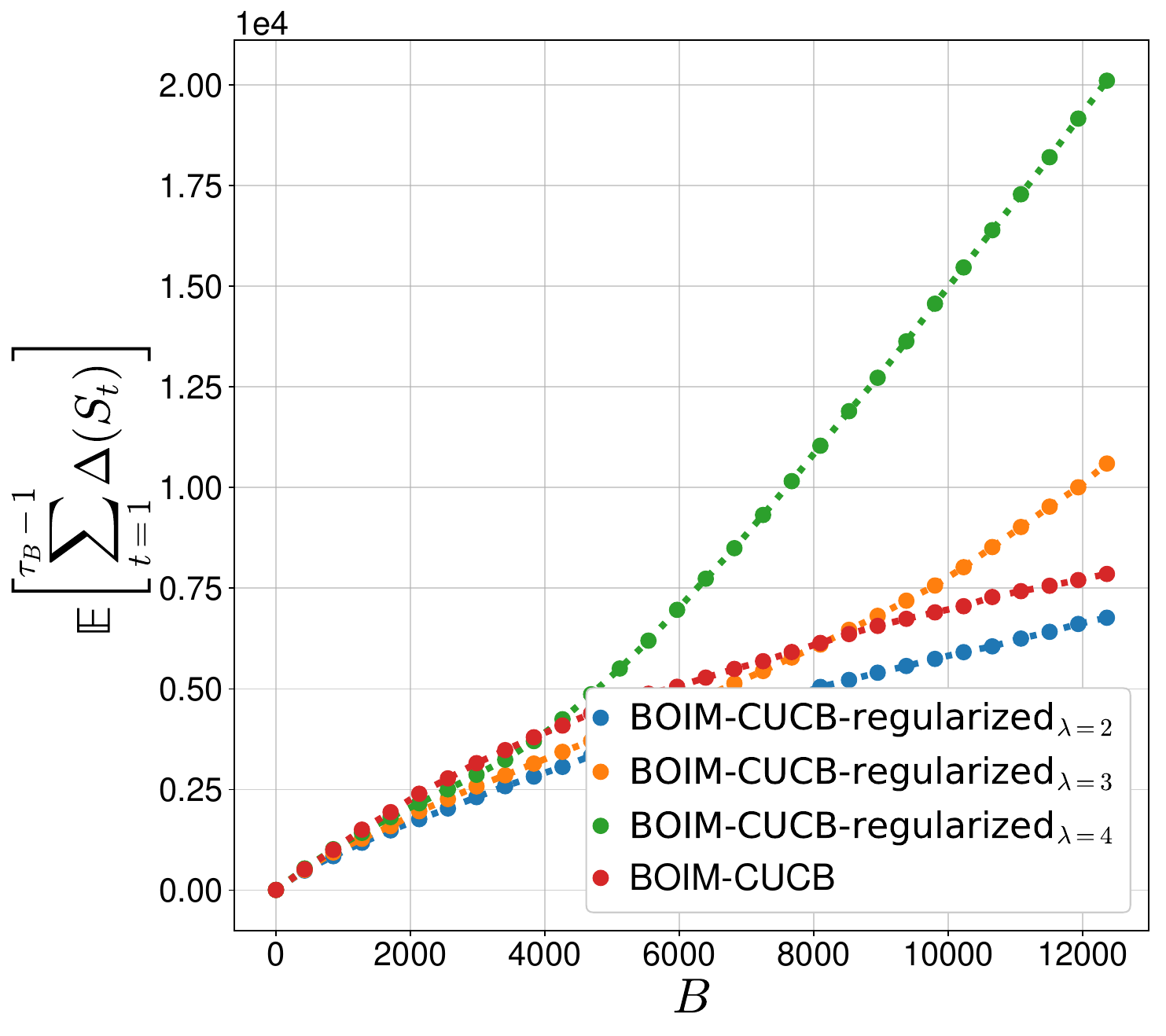}
\includegraphics[trim={3.28cm 0 0 0},clip,scale=.14]{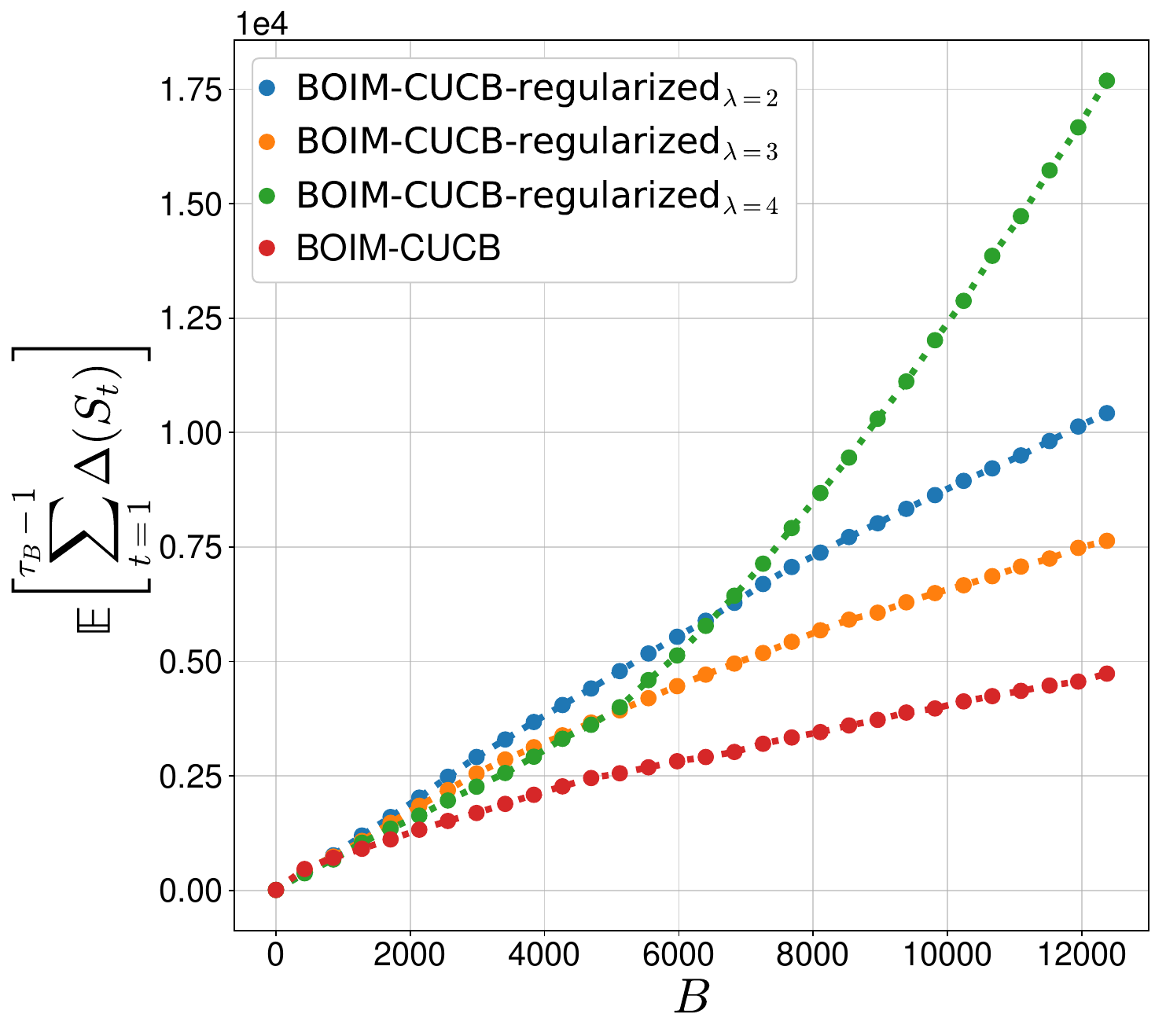}
\caption{Regret curves on five different problem instances, with respect to the budget $B$ (expectation computed by averaging over 10 independent simulations).}
\label{exp:regretbis}
\end{center}
\end{figure}

\section{Proof of Theorem~\ref{thm:jensen}}
\label{app:jensen}
\subsection{Preliminaries}
If we let $p_{S\reach{} S'}(\bw)\triangleq \PP{\sset{i\in V,~S\reach{\bW} i}=S'}$, then another expression is
\begin{align*}&\bonus_4(S;\bw)=\sum_{S'\supset S}p_{S\reach{} S'}(\bw)\underbrace{{\abs{V}\sqrt{\delta\pa{t}\sum_{i\in S'}\frac{d_i}{\counterw{i}{t-1}}}}}_{g(S')}\\&
=\sum_{k\geq 0} \pa{g(S_{k+1}')-g(S_{k}')} \sum_{S'\notin \sset{S_0',\dots,S_{k}'}}p_{S\reach{} S'}(\bw).
\end{align*}
where $g(S)=g(S_1')\leq g(S_2')\leq \dots $ and $S'_{0}=\emptyset$.

Since this bonus shall be used with $\bar\bw_{t-1}$, we need a smoothness inequality to link $p_{S\reach{} S'}(\bar\bw_{t-1})$ to $p_{S\reach{} S'}(\bw^\star)$. We prove here the following such inequality.
\begin{proposition}\label{prop:jensen} For all $S\subset V$, all $\bw,\bw'\in [0,1]^E$ and all collection of subsets of vertices $\cS$, we have
$$\abs{\sum_{S'\in \cS}\!\pa{p_{S\reach{} S'}(\bw)\!-\!p_{S\reach{} S'}(\bw')}}\!\leq\!\! \sum_{ij\in E}\!p_i(S;\!\bw)\abs{w'_{ij}\!-\!w_{ij}}.$$
\end{proposition}
\begin{proof}
We assume w.l.o.g. that $\bw'\geq \bw$.
We consider the random graph $G_{\bW}=\pa{V,\sset{ij\in E, W_{ij}=1}}$, where  $\bW\sim\otimes_{ij \in E}\Bernoulli\pa{w_{ij}}$. We build $G_{\bW'}$ from $G_{\bW}$ by adding edges $ij$ independently with probability $\frac{w'_{ij}-w_{ij}}{1-w_{ij}}$ for each $ij$ that is not an edge in $G_{\bW}$. 
Now, see that
$$\sum_{S'\in \cS} p_{S\reach{} S'}(\bw)=\PP{S\reach{\bW} S'}-\sum_{S'\notin \cS} p_{S\reach{} S'}(\bw),$$
where $S\reach{\bW} S'$ means $S\reach{\bW} i$ for all $i\in S'$. Thus, 
\begin{align*}
    0\leq p_{S\reach{} S'}(\bw')-p_{S\reach{} S'}(\bw) &= \PP{S\reach{\bW'} S'}-\PP{S\reach{\bW} S'} -\pa{\sum_{S'\notin \cS} p_{S\reach{} S'}(\bw')-\sum_{S'\notin \cS} p_{S\reach{} S'}(\bw)}
    \\&\leq \PP{S\reach{\bW'} S'}-\PP{S\reach{\bW} S'}
    \\&= \PP{S\reach{\bW'} S'~\text{but not}~S\reach{\bW} S'}
    \\&\leq \PP{\text{There is an edge }ij \text{ s.t. }  S\reach{\bW} i,~W_{ij}=0,\text{ and }W'_{ij}=1}. 
\end{align*}
The last inequality is by noticing that if $S\reach{\bW'} S'~\text{but not}~S\reach{\bW} S'$, then there must be a edge $ij$ accessible from $S$ in $G_{\bW'}$ such that $W_{ij}=0$ and $W'_{ij}=1$. Taking the first such edge $ij$ (watching the contagion spread from $S$ step by step), we see that $ij$ must be accessible from $S$ in $G_{\bW}$ as well (since otherwise there's a previous accessible edge $k\ell$ that verifies $W_{k\ell}=0$ and $W'_{k\ell}=1$).

We have that $  S\reach{\bW} i$ is independent from $W_{ij},W'_{ij}$. Since $\PP{W_{ij}=0,\text{ and }W'_{ij}=1}=(1-w_{ij})\frac{w'_{ij}-w_{ij}}{1-w_{ij}}=w'_{ij}-w_{ij}$, we have
$$\PP{\text{There is an edge }ij \text{ s.t. }  S\reach{\bW} i,~W_{ij}=0,\text{ and }W'_{ij}=1}\leq \sum_{ij\in E}p_i(S;\bw)\pa{w'_{ij}-w_{ij}}.$$
\end{proof}

\subsection{Main proof of Theorem~\ref{thm:jensen}}
\begin{proof}
We apply a similar analysis as above. When all the events hold, the same analysis gives
\[\Delta\pa{S_t}\leq 2\lambda^\star\alpha \sum_{i\in S_t} 1\wedge 2\sqrt{\frac{1.5\log\pa{t}}{\counterc{i}{t-1}}} + 4\bonus_4\pa{S_t;\vmeanw{t-1}},\]
and the first term can be handled in the same way. 
The second term can be analyzed in the following way: Using Proposition~\ref{prop:jensen} with $\cS=\sset{S'\subset V,~S'\notin \sset{S'_0,\dots,S'_k}}$, we get
\begin{align*}4\bonus_4\pa{S_t;\vmeanw{t-1}}&\leq 4\bonus_4\pa{S_t;\bw^\star} + \sum_{k\geq 0} \pa{g(S_{k+1}')-g(S_{k}')} \frac{1}{\abs{V}}\bonus_4\pa{S_t;\bw^\star}\\&=\pa{4+\sqrt{\delta(t)\sum_{i\in V}\frac{d_i}{\counterw{i}{t-1}}}}\bonus_4\pa{S_t;\bw^\star}\\
&\leq 5\bonus_4\pa{S_t;\bw^\star},\end{align*}
where the last inequality uses the event \[\mathfrak{R}_{t}\triangleq \sset{\forall i\in V, \counterw{i}{t-1}\geq  \abs{E}\delta\pa{t}}.\] 

Relying on Theorem~\ref{thm:jesen}, we can deal with this last term and obtain a term of order
$$ \delta{B/c_0^\star}\sum_{i\in V}\frac{\abs{V}^2d_i\log^2\pa{\abs{E}}}{\Delta_{i,\min}}.$$
\end{proof}

\section{Proof of Proposition~\ref{prop:greedy_bis}}\label{app:greedy_bis}
\begin{proof}
There are two possibilities for $S^\star$: either $\EE{\be_{S^\star\cup\sset{0}}\transpose \bc^\star}<b$, or $\EE{\be_{S^\star\cup\sset{0}}\transpose \bc^\star}=b$. In the first case, we know 
 that $S^\star$ is not random. Indeed, if it is not the case,  then
 $\frac{\EE{\sigma(S^\star)}}{\EE{\be_{S^\star\cup\sset{0}}\transpose \bc^\star}}$ is a convex combination of some $\frac{{\sigma(S)}}{{\be_{S\cup\sset{0}}\transpose \bc^\star}}$ for $S$ in the support of the distribution of $S^\star$. Necessarily, the maximizer (over $S$ in the support) of the ratio is such that $\be_{S\cup\sset{0}}\transpose \bc^\star>b$, since otherwise this maximizer contradicts the definition of $S^\star$. Therefore, increasing the coefficient of this maximizer in the convex combination increases $\EE{\be_{S^\star\cup\sset{0}}\transpose \bc^\star}$, which can thus be set to $b$.
 Since this also increases
 $\frac{\EE{\sigma(S^\star)}}{\EE{\be_{S^\star\cup\sset{0}}\transpose \bc^\star}}$, we get a contradiction since we improved the solution $S^\star$ while still satisfying the constraint.
 \begin{itemize}
     \item 
 Consider the first case. 
We have as for the proof of Proposition~\ref{prop:greedy}, that
\[\pa{1-e^{-1}}\frac{\sigma\pa{S^\star}}{\be_{S^\star\cup\sset{0}}\transpose \bc^\star }\leq \frac{\pa{1-\beta}\sigma\pa{S_{\ell}} + \beta \sigma\pa{S_{\ell+1}}}{\pa{1-\beta}\be_{S_\ell\cup\sset{0}}\transpose\bc^\star + \beta \be_{S_{\ell+1}\cup\sset{0}}\transpose\bc^\star}\CommaBin\]
where $\ell\in\sset{0,1,\dots,\abs{V}-1}$ is such that $\be_{S_\ell}\transpose\bc\leq \be_{S^\star}\transpose\bc\leq \be_{S_{\ell+1}}\transpose\bc$, and $\be_{S^\star\cup\sset{0}}\transpose \bc^\star =\pa{1-\beta}\be_{S_\ell\cup\sset{0}}\transpose\bc^\star + \beta \be_{S_{\ell+1}\cup\sset{0}}\transpose\bc^\star$.
In the case $S_\ell$ has a greater ratio than $S_{\ell+1}$, it is chosen by our algorithm and has the desired approximation. In the case $S_{\ell+1}$ has the better ratio, it is chosen if its cost is lower than $b$. If its cost is greater than~$b$, then $\ell+1=j$ and the algorithm chooses $S_\ell$ with some probability $1-\beta'$ and $S_{\ell+1}$ with probability $\beta'$. The goal is to show that the coefficient $\beta'$ we use for $S_{\ell+1}$ is greater than $\beta$. This must be the case since $\pa{1-\beta'}\be_{S_\ell\cup\sset{0}}\transpose\bc^\star + \beta' \be_{S_{\ell+1}\cup\sset{0}}\transpose\bc^\star=b>\be_{S^\star\cup\sset{0}}\transpose \bc^\star = \pa{1-\beta}\be_{S_\ell\cup\sset{0}}\transpose\bc^\star + \beta \be_{S_{\ell+1}\cup\sset{0}}\transpose\bc^\star.$ 
\item For the second case, we let $S$ be the output of the Algorithm~1 considered by \citet{wang2020fast}. We thus have from their Theorem~1 that
$\pa{1-e^{-1}}\EE{\sigma\pa{S^\star}}\leq \EE{\sigma(S)}$.
Since $\EE{\be_{S\cup\sset{0}}\transpose\bc^\star}=\EE{\be_{S^\star\cup\sset{0}}\transpose\bc^\star}=b$, we have
$$\pa{1-e^{-1}}\frac{\EE{\sigma\pa{S^\star}}}{\EE{\be_{S^\star\cup\sset{0}}\transpose\bc^\star}}\leq \frac{\EE{\sigma(S)}}{\EE{\be_{S\cup\sset{0}}\transpose\bc^\star}}\cdot$$
If the expected cost of the output $S'$ of our algorithm is $b$, then both algorithms coincides and we have the desired result. Else, we have that $S'$ maximizes the ratio over $\sset{S_0,\dots,S_{j}}$, which contains the support of $S$ (that is $\sset{S_{j-1},S_{j}}$), so the ratio evaluated at $S'$ is greater than $\frac{\EE{\sigma(S)}}{\EE{\be_{S\cup\sset{0}}\transpose\bc^\star}}$, giving again the desired result.
 \end{itemize}
\end{proof}

\section{Generalities on combinatorial multi-armed bandits}
In this section, $i$ represent an ``arm", i.e., an edge in our OIM context. $A_t$ is the random set of edges that are triggered at round $t$. Here, the horizon $T$ can be random.
Finally, $b_i(S_t)$ is simply some non-negative function (for our $\bonus_4$, this is $\abs{V}$ times the square root of the out-degree) and $m_i$ is the maximum number of edges that can be reached when $i$ is activated. The following theorem is based on \citet{perrault2020covariance}, Theorem~4.
\begin{theorem}[Regret bound for $\ell_2$-bonus, with expectation outside the norm]\label{thm:jesen}For all $i\in [n]$, let $\pa{\alpha_i,\beta_{i,T}}\in [1/2,1)\times \R_+$.
For $t\geq 1$,
consider the event
$$\mathfrak{A}_t\triangleq\sset{\Delta_t\leq
\EEcc{\norm{\sum_{i\in A_t,N_{i,t-1}>0}\frac{b_i(S_t)\beta_{i,T}^{\alpha_i}\be_i}{N_{i,t-1}^{\alpha_i}}}_2}{\cF_t}},$$

Then, if $\sset{t\leq T}\in \cF_t$, we have $$\EE{\sum_{t=1}^T\II{\mathfrak{A}_t}\Delta_t}\leq \sum_{i\in [n]}4\log_2(4\sqrt{m_i})\max_{S\in \cS,~p_i(S)>0} b_i(S)^{\frac{1}{\alpha_i}}{\EE{\beta_{i,T}}\eta_i},$$
where
\begin{align*}\eta_i=
 \left\{
    \begin{array}{ll}
        {32\log_2(4\sqrt{m_i})}{\Delta_{i,\min}^{-1}} &\mbox{if } \alpha_i=1/2 \\
        2^{\frac{2}{\alpha_i}}\pa{\pa{1-2^{\frac{1}{\alpha_i}-2}}\pa{1-\alpha_i}\Delta_{i,\min}^{\frac{1-\alpha_i}{\alpha_i}}}^{-1} & \mbox{if }  1/2<\alpha_i<1.
    \end{array}
\right.
\end{align*}
\end{theorem}

\begin{proof} 
Let $t\geq 1$. With a first reverse amortisation, we start by restricting the set of possibles for $A_t$ by only taking those whose error is at least twice as large as $\Delta_t$: assuming that $\mathfrak{A}_t$ holds, we have
\begin{align*}
    \Delta_t &\leq   \EEcc{2\norm{\sum_{i\in A_t,N_{i,t-1}>0}\frac{b_i(S_t)\beta_{i,T}^{\alpha_i}\be_i}{N_{i,t-1}^{\alpha_i}}}_2-\Delta_t}{\cF_t}
    \\&\leq \EEcc{\II{\norm{\sum_{i\in A_t,N_{i,t-1}>0}\frac{2b_i(S_t)\beta_{i,T}^{\alpha_i}\be_i}{N_{i,t-1}^{\alpha_i}}}_2\geq\Delta_t}\norm{\sum_{i\in A_t,N_{i,t-1}>0}\frac{2b_i(S_t)\beta_{i,T}^{\alpha_i}\be_i}{N_{i,t-1}^{\alpha_i}}}_2}{\cF_t}
\end{align*}
We now
define $$\Lambda\pa{A_t}\triangleq {\norm{\sum_{i\in A_t,N_{i,t-1}>0}\frac{2b_i(S_t)\beta_{i,T}^{\alpha_i}\be_i}{N_{i,t-1}^{\alpha_i}}}_2},$$
and have for any $j\in A_t$ that 
\begin{align}\Lambda\pa{A_t}\geq {\frac{2b_j(S_t)\beta_{j,T}^{\alpha_j}}{N_{j,t}^{\alpha_j}}}.\label{LowerLambda_t}\end{align}
Then, we can write:  
\begin{align*}
    \Lambda\pa{A_t}&=-\Lambda\pa{A_t}+{\norm{\sum_{i\in A_t,N_{i,t-1}>0}\frac{4b_i(S_t)\beta_{i,T}^{\alpha_i}\be_i}{N_{i,t-1}^{\alpha_i}}}_2}
    \\
    &={
    -\norm{\sum_{i\in A_t}\frac{\Lambda\pa{A_t} \be_i}{\norm{\be_{A_t}}_2}}_2+\norm{\sum_{i\in A_t,N_{i,t-1}>0}\frac{4b_i(S_t)\beta_{i,T}^{\alpha_i}\be_i}{N_{i,t-1}^{\alpha_i}}}_2}
        \\
    &\leq {
    \norm{\sum_{i\in A_t,N_{i,t-1}>0}\pa{\frac{4b_i(S_t)\beta_{i,T}^{\alpha_i}}{N_{i,t-1}^{\alpha_i}}-\frac{\Lambda\pa{A_t}}{\norm{\be_{A_t}}_2}}^+\be_i }_2 }
          \\
              &= {
    \norm{\sum_{i\in A_t,N_{i,t-1}>0}\pa{\frac{4b_i(S_t)\beta_{i,T}^{\alpha_i}}{N_{i,t-1}^{\alpha_i}}-\frac{\Lambda\pa{A_t}}{\norm{\be_{A_t}}_2}}^+\II{\Lambda\pa{A_t}\geq \frac{2b_i(S_t)\beta_{i,T}^{\alpha_i}}{N_{i,t-1}^{\alpha_i}} }\be_i }_2} &\text{Using } \eqref{LowerLambda_t}
          \\
    &\leq {\norm{\sum_{i\in A_t,N_{i,t-1}>0}\II{2\Lambda\pa{A_t}\geq\frac{4b_i(S_t)\beta_{i,T}^{\alpha_i}}{N_{i,t-1}^{\alpha_i}}\geq \frac{\Lambda\pa{A_t}}{\norm{\be_{A_t}}_2}}\frac{4b_i(S_t)\beta_{i,T}^{\alpha_i}\be_i}{N_{i,t-1}^{\alpha_i}} }_2}.
\end{align*}
We now consider the following partition of the set of indices:
$$\II{i\in A_t,N_{i,t-1}>0,~2\Lambda\pa{A_t}\geq\frac{4b_i(S_t)\beta_{i,T}^{\alpha_i}}{N_{i,t-1}^{\alpha_i}}\geq \frac{\Lambda\pa{A_t}}{\norm{\be_{A_t}}_2}}\subset \bigcup_{k=0}^{\ceil{\log_2\pa{\norm{\be_{A_t}}_2}}} J_{k,t}, $$
where for all integer $1\leq k\leq {\ceil{\log_2\pa{\norm{\be_{A_t}}_2}}} $,
$$J_{k,t}\triangleq \sset{i\in A_t,N_{i,t-1}>0,~2^{1-k}\Lambda\pa{A_t}\geq\frac{4b_i(S_t)\beta_{i,T}^{\alpha_i}}{N_{i,t-1}^{\alpha_i}}\geq {2^{-k}\Lambda\pa{A_t}}}.$$
We bound $\Lambda\pa{A_t}^2$ as
\begin{align*}\Lambda\pa{A_t}^2&\leq {\norm{\sum_{i\in A_t,N_{i,t-1}>0}\II{2\Lambda\pa{A_t}\geq\frac{4b_i(S_t)\beta_{i,T}^{\alpha_i}}{N_{i,t-1}^{\alpha_i}}\geq \frac{\Lambda\pa{A_t}}{\norm{\be_{A_t}}_2}}\frac{4b_i(S_t)\beta_{i,T}^{\alpha_i}\be_i}{N_{i,t-1}^{\alpha_i}} }_2^2}
\\&= {\sum_{k=0}^{\ceil{\log_2\pa{\norm{\be_{A_t}}_2}}}\norm{\sum_{i\in J_{k,t}}{\frac{4b_i(S_t)\beta_{i,T}^{\alpha_i}\be_i}{N_{i,t-1}^{\alpha_i}}}}_2^2}
\\&\leq {
\sum_{k=0}^{\ceil{\log_2\pa{\norm{\be_{A_t}}_2}}}{2^{2-2k}}\Lambda(A_t)^2\norm{\be_{J_{k,t}} }_2^2} .
\end{align*}
So there exists one integer $k_t$ such that $\abs{J_{k_t,t}}=\norm{\be_{J_{k_t,t}}}^2_2\geq 2^{2k_t-2}\pa{1+\ceil{\log_2\pa{\norm{\be_{A_t}}_2}}}^{-1}$. 

\begin{align*}{\sum_{t=1}^T\II{\mathfrak{A}_t}\Delta_t}&\leq\sum_{t=1}^T\EEcc{\II{\Lambda(A_t)\geq \Delta_t}\Lambda(A_t)}{\cF_t}
\\&\leq
\sum_{t=1}^T\EEcc{\sum_{k=0}^{\ceil{\log_2\pa{\norm{\be_{A_t}}_2}}}\II{k_t=k,~\Lambda(A_t)\geq \Delta_t}\Lambda(A_t)}{\cF_t}
\\&\leq\!
\sum_{t=1}^T\!\EEcc{\!\sum_{k=0}^{\ceil{\log_2\pa{\norm{\be_{A_t}}_2}}}\!\!\!\!\!\II{k_t\!=\!k,~\Lambda(A_t)\!\geq\! \Delta_t}\frac{\sum_{i\in [n]}\II{i\in J_{k,t}}}{2^{2k-2}\pa{1\!+\!\ceil{\log_2\pa{\norm{\be_{A_t}}_2}}}^{-1}}\Lambda(A_t)}{\cF_t}
\\&\leq
\sum_{t=1}^T\!\sum_{i=1}^n\!\EEcc{\!\sum_{k=0}^{\ceil{\log_2\pa{\norm{\be_{A_t}}_2}}}\!\frac{\II{i\!\in\! A_t,~0<N_{i,t-1}^{\alpha_i}\!\leq\! \frac{2^{k+2}b_i(S_t)\beta_{i,T}^{\alpha_i}}{\Lambda(A_t)},~\Lambda(A_t)\!\geq\! \Delta_t}}{2^{2k-2}\pa{1+\ceil{\log_2\pa{\norm{\be_{A_t}}_2}}}^{-1}}\Lambda(A_t)}{\cF_t}.
\end{align*}
Taking the expectation of the above, and using $\sset{t\leq T}\in\cF_t $, we have the bound
\begin{align*}
&\EE{\sum_{t=1}^T\II{\mathfrak{A}_t}\Delta_t}\leq\sum_{i=1}^n\!\EE{\sum_{t=1}^T\!\!{\sum_{k=0}^{\ceil{\log_2\pa{\norm{\be_{A_t}}_2}}}\!\frac{\II{i\!\in\! A_t,~0<N_{i,t-1}^{\alpha_i}\!\leq\! \frac{2^{k+2}b_i(S_t)\beta_{i,T}^{\alpha_i}}{\Lambda(A_t)},~\Lambda(A_t)\!\geq\! \Delta_t}}{2^{2k-2}\pa{1+\ceil{\log_2\pa{\norm{\be_{A_t}}_2}}}^{-1}}\Lambda(A_t)}}
\\&\leq
\sum_{i=1}^n\!\!\sum_{k=0}^{\ceil{\log_2\pa{\sqrt{m_i}}}}\!\!\!\EE{\frac{1\!+\!\ceil{\log_2\pa{\sqrt{m_i}}}}{2^{2k-2}}\underbrace{\sum_{t=1}^T\!{{\II{i\!\in\! A_t,~0<N_{i,t-1}^{\alpha_i}\!\leq\! \frac{2^{k+2}b_i(S_t)\beta_{i,T}^{\alpha_i}}{\Lambda(A_t)},~\Lambda(A_t)\!\geq\! \Delta_t}}\Lambda(A_t)}}_{\numterm{sumcounter}_{i,k}}}
.
\end{align*}
Applying Proposition~\ref{prop:sumcounter} gives
$$\eqref{sumcounter}_{i,k}\leq \frac{\max_{S\in \cS,~p_i(S)>0} b_i(S)^{\frac{1}{\alpha_i}}\beta_{i,T}2^{\frac{k+2}{\alpha_i}}}{1-\alpha_i}\Delta_{i,\min}^{1-1/\alpha_i}  ,$$
So  using $\ceil{\log_2\pa{\sqrt{m_i}}}+1\leq \log_2\pa{4\sqrt{m_i}}$,
we get
$$\EE{\sum_{t=1}^T\II{\mathfrak{A}_t}\Delta_t}\leq \sum_{i\in [n]}4\log_2(4\sqrt{m_i})\max_{S\in \cS,~p_i(S)>0} b_i(S)^{\frac{1}{\alpha_i}}{\EE{\beta_{i,T}}\eta_i},$$
where
\begin{align*}\eta_i=
 \left\{
    \begin{array}{ll}
        {32\log_2(4\sqrt{m_i})}{\Delta_{i,\min}^{-1}} &\mbox{if } \alpha_i=1/2 \\
        2^{\frac{2}{\alpha_i}}\pa{\pa{1-2^{\frac{1}{\alpha_i}-2}}\pa{1-\alpha_i}\Delta_{i,\min}^{\frac{1-\alpha_i}{\alpha_i}}}^{-1} & \mbox{if }  1/2<\alpha_i<1.
    \end{array}
\right.
\end{align*}
\end{proof}
\begin{proposition} Let $i\in [n]$ and $f_i: \R_+\to \R_+$ be a non increasing function, integrable on an interval $[\delta_{i,\min},\delta_{i,\max}]\subset \R_+^\star$. Then for any sequence of real numbers $\pa{\delta_t}\in \pa{[\delta_{i,\min},\delta_{i,\max}]\cup\sset{0}}^T$,
$$\sum_{t=1}^T{\II{i\in A_t,~ 1\leq N_{i,t-1}\leq f_i(\delta_t)}}{\delta_t} \leq
{f_i(\delta_{i,\min})}\delta_{i,\min}+
\int_{\delta_{i,\min}}^{\delta_{i,\max}}{f_i(x)\emph{d}x}.$$
In particular, \\
\begin{itemize}
\item If $f_i(x)=\beta_{i,T}x^{-1/\alpha_i}$, $\alpha_i\in (0,1)$ and $\beta_{i,T}\geq 0$, then 
\begin{align*}\sum_{t=1}^T{\II{i\in A_t,~ 1\leq N_{i,t-1}\leq f_i(\delta_t)}}{\delta_t} &\leq
\delta_{i,\min}^{1-1/\alpha_i}\frac{\beta_{i,T}}{1-\alpha_i}-\delta_{i,\max}^{1-1/\alpha_i}\frac{\alpha_i\beta_{i,T}}{1-\alpha_i}\\&\leq \delta_{i,\min}^{1-1/\alpha_i}\frac{\beta_{i,T}}{1-\alpha_i}
.\end{align*}

\item If $f_i(x)=\beta_{i,T}x^{-1}$, $\beta_{i,T}\geq 0$, then 
$$\sum_{t=1}^T{\II{i\in A_t,~ 1\leq N_{i,t-1}\leq f_i(\delta_t)}}{\delta_t} \leq\beta_{i,T}\pa{1+\log\pa{\frac{\delta_{i,\max}}{\delta_{i,\min}}}}
.$$
\end{itemize}
  \label{prop:sumcounter}
 \end{proposition}
\begin{proof}
Consider $\delta_{i,\max}=\delta_{i,1}\geq \delta_{i,2}\geq \dots \geq\delta_{i,K_i}=\delta_{i,\min}$ being all possible values for $\delta_t$ when $\delta_t\neq 0$.
We define a dummy gap $\delta_{i,0}=\infty$ and let $f_i\pa{\delta_{i,0}}=0$.
In \eqref{rangebreak}, we look at times $t$ where $\delta_t\neq 0$ and first break the range $(0,f_i(\delta_t)]$ of the counter $N_{i,t-1}$ into sub intervals: $$(0,f_i(\delta_t)]= (f_i(\delta_{i,0}),f_i(\delta_{i,1})]\cup\dots\cup (f_i(\delta_{i,k_t-1}),f_i(\delta_{i,k_t})],$$ where $k_t$ is the index such that $\delta_{i,k_t} = \delta_t$. {This index $k_t$ exists by assumption that the subdivision contains all possible values for $\delta_t$ when $\delta_t\neq 0$. Notice that in \eqref{rangebreak}, we do not explicitly use $k_t$, but instead sum over all $k\in [K_i]$ and filter against the event $\sset{\delta_{i,k}\geq \delta_t}$, which is equivalent to summing over $k\in [k_t].$}  

\begin{align}
&\sum_{t=1}^T{\II{i\in A_t,~ N_{i,t-1}\leq f_i(\delta_t)}}{\delta_t}\nonumber
\\&=\label{rangebreak}
\sum_{t=1}^T\sum_{k=1}^{K_i}\II{i\in A_t,~f_i(\delta_{i,k-1})< N_{i,t-1}\leq f_i(\delta_{i,k}),\delta_{i,k}\geq \delta_t}{\delta_t}.
\end{align}
Over each event that $N_{i,t-1}$ belongs to the interval $(f_i(\delta_{i,k-1}),f_i(\delta_{i,k})]$, we upper bound the  gap $\delta_t$ by $\delta_{i,k}$. 

\begin{align}
\eqref{rangebreak}&\leq\label{uppergap}
\sum_{t=1}^T\sum_{k=1}^{K_i}\II{i\in A_t,~f_i(\delta_{i,k-1})< N_{i,t-1}\leq f_i(\delta_{i,k}),\delta_{i,k}\geq \delta_t}{\delta_{i,k}}.
\end{align}
Then, we further upper bound the summation by adding events that $N_{i,t-1}$ belongs to the remaining intervals
$(f_i(\delta_{i,k-1}),f_i(\delta_{i,k})]$ for $k_t<k\leq K_i$, associating them to a suffered gap $\delta_{i,k}$. This is equivalent to removing the filtering against the event $\sset{\delta_{i,k}\geq \delta_t}$. 

\begin{align}
\eqref{uppergap}&\leq\label{addremainingintervals}
\sum_{t=1}^T\sum_{k=1}^{K_i}\II{i\in A_t,~f_i(\delta_{i,k-1})< N_{i,t-1}\leq f_i(\delta_{i,k})}{\delta_{i,k}}.\end{align}
Now, we invert the summation over $t$ and the one over $k$. 

\begin{align}
\eqref{addremainingintervals}&=\label{sumswitch}
\sum_{k=1}^{K_i}\sum_{t=1}^T\II{i\in A_t,~f_i(\delta_{i,k-1})< N_{i,t-1}\leq f_i(\delta_{i,k})}{\delta_{i,k}}.
\end{align}
For each $k\in[K_i]$, the number of times $t\in [T]$ that the counter $N_{i,t-1}$ belongs to $(f_i(\delta_{i,k-1}),f_i(\delta_{i,k})]$ can be upper bounded by the number of integers in this interval. This is due to the event $\sset{i\in A_t}$, imposing that $N_{i,t-1}$ is incremented, so $N_{i,t-1}$ cannot be worth the same integer for two different times $t$ satisfying $i\in A_t$. We use the fact that for all $x,y\in \R$, $x\leq y$, the number of integers in the interval $(x,y]$ is exactly $\floor{y}-\floor{x}$. 

\begin{align}
\eqref{sumswitch}&\leq\label{intcount}
\sum_{k=1}^{K_i}\pa{\floor{f_i(\delta_{i,k})}-\floor{f_i(\delta_{i,k-1})}}{\delta_{i,k}}.
\end{align}
We then simply expand the summation, and some terms are cancelled (remember that $f_i\pa{\delta_{i,0}}=0$).

\begin{align}
\eqref{intcount}&=\label{expand}
\floor{f_i(\delta_{i,K_i})}\delta_{i,K_i}+
\sum_{k=1}^{K_i-1}\floor{f_i(\delta_{i,k})}\pa{\delta_{i,k}-\delta_{i,k+1}}
\end{align}
We use $\floor{x}\leq x$ for all $x\in \R$. Finally, we recognize a right Riemann sum, and use the fact that $f_i$ is non increasing
to upper bound each $f_i(\delta_{i,k})\pa{\delta_{i,k}-\delta_{i,k+1}}$ by $\int_{\delta_{i,k+1}}^{\delta_{i,k}}f_i(x)\text{d}x$, for all $k\in [K_i-1]$.
\begin{align}
\eqref{expand}&\leq\label{upperfloor}
{f_i(\delta_{i,K_i})}\delta_{i,K_i}+
\sum_{k=1}^{K_i-1}{f_i(\delta_{i,k})}\pa{\delta_{i,k}-\delta_{i,k+1}}
\\&\leq\label{Riemann}
{f_i(\delta_{i,K_i})}\delta_{i,K_i}+
\int_{\delta_{i,K_i}}^{\delta_{i,1}}{f_i(x)\text{d}x}.
\end{align}
\end{proof}

\end{document}